\documentclass{article}

\usepackage[numbers]{natbib}

    \usepackage[main,final]{neurips_2025}

\usepackage[utf8]{inputenc} 
\usepackage[T1]{fontenc}    
\usepackage{hyperref}       
\usepackage{url}            
\usepackage{booktabs}       
\usepackage{amsfonts}       
\usepackage{nicefrac}       
\usepackage{microtype}      
\usepackage{xcolor}         
\usepackage{subcaption}
\usepackage{amsthm}
\usepackage{algorithm} 
\usepackage{algpseudocode}
\usepackage{siunitx}
\title{Adaptive Stochastic Coefficients for Accelerating Diffusion Sampling}

\author{%
  Ruoyu Wang\textsuperscript{1{$\star$}}\quad Beier Zhu\textsuperscript{1,2{$\star$}}\quad Junzhi Li\textsuperscript{3,4}\quad Liangyu Yuan\textsuperscript{5}\quad Chi Zhang\textsuperscript{1${\dag}$} \\
  $^{1}$AGI Lab, Westlake University \quad  $^{2}$Nanyang Technological University \\
  $^{3}$University of Chinese Academy of Sciences \\ $^{4}$Institute of Software, Chinese Academy of Sciences\\ 
  $^{5}$ Tongji University \\
  \texttt{\{wangruoyu71,chizhang\}@westlake.edu.cn}\\
  \texttt{beier.zhu@ntu.edu.sg}\\
  \texttt{lijunzhi25@mails.ucas.ac.cn}\\
  \texttt{liangyuy001@gmail.com}
}

\usepackage{graphicx}  
\usepackage{float}     
\usepackage{array}     
\usepackage{multirow}  
\usepackage{rotating}  
\usepackage{booktabs}  
\usepackage{caption}   
\usepackage{amsmath}
\usepackage{cleveref}
\usepackage{wrapfig}
\usepackage[table]{xcolor}
\usepackage{colortbl}
\usepackage{wrapfig}
\usepackage{bm}
\usepackage{pifont}
\definecolor{lightCyan}{rgb}{0.925,1,1}
\usepackage{enumitem}

\newcommand{\ours}{\texttt{AdaSDE}}

\definecolor{tabhighlight}{HTML}{e5e5e5}

\newtheorem{theorem}{Theorem}

\newtheorem{lemma}{Lemma}

\newtheorem{remark}{Remark}

\newcommand{\R}[1]{{%
    \textbf{%
        \ifstrequal{#1}{1}{\textcolor{red}{R#1}}{%
        \ifstrequal{#1}{2}{\textcolor{blue}{R#1}}{%
        \ifstrequal{#1}{3}{\textcolor{magenta}{R#1}}{%
        \ifstrequal{#1}{4}{\textcolor{teal}{R#1}}{%
                           \textcolor{cyan}{R#1}%
        }}}}%
    }%
}}

\def\eqref#1{equation~\ref{#1}}

\def\1{\bm{1}}

\def\rvx{{\mathbf{x}}}
\def\rvy{{\mathbf{y}}}

\DeclareMathAlphabet{\mathsfit}{\encodingdefault}{\sfdefault}{m}{sl}
\SetMathAlphabet{\mathsfit}{bold}{\encodingdefault}{\sfdefault}{bx}{n}

\begin{document}

\maketitle

\begin{NoHyper}
\def\thefootnote{$\star$}\footnotetext{Equal contribution. $^\dag$Corresponding author.}
\end{NoHyper}

\begin{abstract}
Diffusion-based generative processes,  formulated as differential equation solving, frequently balance computational speed with sample quality.    Our theoretical investigation of ODE- and SDE-based solvers reveals complementary weaknesses:  ODE solvers accumulate irreducible gradient error along deterministic trajectories, while SDE methods suffer from amplified discretization errors when the step budget is limited. Building upon this insight, we introduce \ours, a novel single-step SDE solver that aims to unify the efficiency of ODEs with the error resilience of SDEs. Specifically, we introduce a single per-step learnable coefficient, estimated via lightweight distillation,  which dynamically regulates the error correction strength to accelerate diffusion sampling. Notably, our framework can be integrated with  existing  solvers to enhance their capabilities. Extensive experiments demonstrate state-of-the-art performance: at 5 NFE, \ours~ achieves FID scores of $4.18$ on CIFAR-10, $8.05$ on FFHQ and $6.96$ on LSUN Bedroom. Codes are available in \url{https://github.com/WLU-wry02/AdaSDE}. 

\end{abstract}
\section{Introduction}
\label{sec:intro}

Diffusion Models (DMs)~\cite{sohl2015deep,ho2020denoisingdiffusionprobabilisticmodels,NEURIPS2021_49ad23d1,NEURIPS2022_ec795aea,rombach2022high} have revolutionized generative modeling, achieving state-of-the-art performance across a broad range of applications~\cite{zhao2025realtimemotioncontrollableautoregressivevideo,chen2025detailtrainingfreeenhancertexttoimage,gao2025subjectconsistentposediversetexttoimagegeneration,Lei_2025_CVPR,jin2025tpblend,song2025worldforgeunlockingemergent3d4d,zhang2025videorepalearningphysicsvideo,TDM,DiQDiff}. Rooted in non-equilibrium thermodynamics, DMs learn to reverse a diffusion process: data are first gradually corrupted by Gaussian noise in a forward phase, and then reconstructed from pure noise through a learned reverse dynamics. This principled design offers stable training and exact likelihood modeling~\cite{song2021scorebasedgenerativemodelingstochastic}, resolving long-standing challenges in earlier approaches, \textit{e.g.}, GANs~\cite{goodfellow2014generativeadversarialnetworks} and VAEs~\cite{kingma2022autoencodingvariationalbayes}.

Recent advances in diffusion models have highlighted the role of differential-equation solvers in balancing sampling speed and generation quality. We first develop a unified error analysis that decomposes the total approximation error into two orthogonal components: \textbf{(1) gradient error}, the discrepancy between the learned score function and the ground-truth score; and \textbf{(2) discretization error}, introduced by time discretization during sampling. Viewed through this lens, existing solvers exhibit complementary behaviors. \textit{Ordinary differential equation (ODE)} based methods offer deterministic trajectories with modest discretization error for low number of function evaluations (NFEs), but their performance is fundamentally constrained by the {irreversible accumulation} of gradient error~\cite{meng2023distillationguideddiffusionmodels,karras2022elucidating,zhou2024fast,lu2022dpm}. In contrast, \textit{stochastic differential equation (SDE)} based methods inject stochasticity that can mitigate gradient error and enhance sample diversity; however, effectively suppressing gradient error in practice usually requires large step counts (\textit{e.g.}, 100–1{,}000 NFEs)~\cite{ho2020denoisingdiffusionprobabilisticmodels,lu2022dpm_plus}. {Hybrid strategies} such as restart sampling\cite{xu2023restartsamplingimprovinggenerative} alternate forward noise injection with backward ODE integration to combine these advantages, yet they still operate in high-NFE regimes. 

Building on the above analysis, we introduce $\ours$, a novel single-step SDE solver that unifies the computational efficiency of ODEs with the error resilience of SDEs under low-NFE budgets. Unlike traditional SDE solvers~\cite{nichol2021improveddenoisingdiffusionprobabilistic,ho2020denoisingdiffusionprobabilisticmodels} that inject fixed noise based on a predetermined time schedule, $\ours$ employs an \emph{adaptive noise injection} mechanism controlled by a learnable stochastic coefficient $\gamma_i$ at each denoising step $i$. 
To effectively optimize ${\gamma_i}$, we further develop a \emph{process-supervision} optimization framework that provides fine-grained guidance at each intermediate step rather than only supervising the final reconstruction. This design is inspired by the observation that diffusion trajectories exhibit consistent low-dimensional geometric structures across solvers and datasets~\cite{chen2024geometricperspectivediffusionmodels,chen2024trajectoryregularityodebaseddiffusion}. By aligning the geometry of the trajectories using $\gamma_i$, $\ours$ reduces gradient error through adaptive stochastic injection, while preserving deterministic efficiency of ODE solvers.

Extensive experiments on both pixel-space  and latent-space DMs demonstrate the superiority of $\ours$. Remarkably, with only 5 NFE, $\ours$ achieves FID scores of $4.18$ on CIFAR-10 \cite{krizhevsky2009learning} and $8.05$ on FFHQ 64$\times$64 \cite{lin2014microsoft}, surpassing the leading AMED-Solver~\cite{zhou2024fast} by 1.8$\times$. Our contributions are threefold:
\begin{itemize}[leftmargin=1.5em]
    \item  We conduct a theoretical comparison of SDE and ODE error dynamics, demonstrating that SDEs offer more robust gradient error control.
    \item  We introduce $\ours$, the first single-step SDE solver that achieves efficient sampling ($<$10 NFEs) by optimizing adaptive $\gamma$-coefficients. Moreover, $\ours$ serves as a universal plug-in module that can enhance existing single-step solvers.
    \item Extensive evaluations on multiple generative benchmarks show that our \ours~achieves state-of-the-art performance with significant FID gains over existing solvers.

\end{itemize}



\section{Related Work}
\label{sec:related}

Recent advancements in accelerating DMs primarily progress along two directions: improved numerical solvers and training-based distillation.

\noindent\textbf{Improved numerical solvers.}
Early studies~\cite{ho2020denoisingdiffusionprobabilisticmodels,nichol2021improveddenoisingdiffusionprobabilistic} accelerated sampling by improving noise-schedule design, and DDIM~\cite{songdenoising} later introduced a non-Markovian formulation that enabled deterministic and much faster sampling.
The establishment of the probability-flow ODE view~\cite{song2021scorebasedgenerativemodelingstochastic} further unified diffusion formulations and paved the way for higher-order numerical schemes and practical preconditioning strategies, exemplified by EDM~\cite{Karras2022edm}. Following this insight, a series of ODE/SDE integrators have emerged to push the accuracy–speed frontier.
For instance, DEIS~\cite{zhangfast}, DPM-Solver~\cite{lu2022dpm}, and DPM-Solver++\cite{lu2022dpm_plus} exploit exponential integration, Taylor expansion, and data-prediction parameterization to achieve robust few-step sampling. Linear multistep variants, including iPNDM~\cite{liupseudo,zhangfast} and UniPC~\cite{zhao2024unipc}, further enable efficient DMs sampling with $\sim$10 NFE.
Hybrid and stochastic extensions extend beyond deterministic solvers: Restart Sampling~\cite{xu2023restartsamplingimprovinggenerative} alternates ODE trajectories with SDE-style noise injection; SA-Solver~\cite{xue2024sa} introduces a training-free stochastic Adams multi-step scheme with variance-controlled noise.

\noindent\textbf{Training-based distillation.}
 Two main paradigms dominate this research direction.
The first is \textit{trajectory approximation}, which uses compact student networks to approximate trajectories generated by teacher models, reducing computational steps. This can be achieved offline: by curating datasets from pre-generated samples~\cite{liu2022flowstraightfastlearning}, or online through progressive distillation that gradually decreases the number of sampling steps~\cite{berthelot2023tractdenoisingdiffusionmodels,meng2023distillationguideddiffusionmodels}.
The second paradigm is \textit{temporal alignment}, which enforces coherence across sampling trajectories by aligning intermediate predictions between adjacent timesteps~\cite{kim2024consistencytrajectorymodelslearning,luo2023latentconsistencymodelssynthesizing}, or by minimizing distributional gaps between real and synthesized data~\cite{sauer2023adversarialdiffusiondistillation,wang2023prolificdreamerhighfidelitydiversetextto3d}.
While these methods improve generation quality and efficiency, they typically require substantial computational resources and complex training protocols, limiting their practicality.
Recent distillation-based solvers—such as AMED~\cite{zhou2024fast}, EPD~\cite{zhu2025distilling}, and D-ODE~\cite{kim2024consistencytrajectorymodelslearning}—achieve few-step sampling through lightweight tuning rather than full retraining. Complementary efforts on time schedule optimization, including LD3~\cite{tong2025learningdiscretizedenoisingdiffusion}, DMN~\cite{DMN}, and GITS~\cite{chen2024trajectoryregularityodebaseddiffusion}, further improve efficiency. While most few-step samplers are rooted in ODE formulations, our approach introduces few-step SDE-driven generation by learning stochastic coefficients under a computationally lightweight objective.

\section{Preliminaries}
\label{sec:Preliminaries}
\subsection{Diffusion Models with Differential Equations}
DMs define a forward process that perturbs data into a  noise distribution, followed by a learned reverse process that inverts this perturbation to generate samples. The forward process is designed as a stochastic trajectory governed by a predefined noise schedule, which can be described by:
\begin{equation}\label{eq:sde}
\mathrm{d}\mathbf{x}=\tfrac{\dot{s}(t)}{s(t)} \mathbf{x}+s(t) \sqrt{2 \sigma(t) \dot{\sigma}(t)} \mathrm{d} \mathbf{w}
\end{equation}
where $\sigma(t)$ is the monotonically increasing noise schedule, and $\mathbf{w}$ denotes a standard Wiener process. This formulation ensures that the marginal distribution $p_t(\mathbf{x})$ at time $t$ corresponds to the convolution of the data distribution $p_0=p_{\text {data }}$ with a Gaussian kernel of variance $\sigma^2(t)$. By selecting a sufficiently large terminal time $T, p_T$ converges to an isotropic Gaussian $\mathcal{N}(\mathbf{0}, \sigma^2(T) \mathbf{I})$, serving as the prior.
Sampling is performed by reversing the forward dynamics through either a reverse-time SDE in Eq.~(\ref{eq:sde}) or an ODE~\cite{song2021scorebasedgenerativemodelingstochastic}:
\begin{equation}\label{eq:2}
      \quad \mathrm{d} \mathbf{x}=-\sigma(t) \dot{\sigma}(t) \nabla_\mathbf{x} \log p_t(\mathbf{x}) \mathrm{d} t .
\end{equation}
Here, the score function $\nabla_\mathbf{x} \log p_t(\mathbf{x})$ is the drift term that guides samples toward high density regions of $p_0$. Following common practice \cite{karras2022elucidating}, the noise schedule is simplified to $\sigma(t)=t$, reducing $\sigma(t) \dot{\sigma}(t)$ to $t$.
A neural network $s_\theta(\mathbf{x}, t)$ is optimized through denoising score matching \cite{song2021scorebasedgenerativemodelingstochastic} to estimate the score function. The training objective minimizes the weighted expectation:
\begin{equation}\label{eq:Loss1}
    \mathbb{E}_{t, \mathbf{x}_0, \mathbf{x}_t}\left[\lambda(t)\left\|s_\theta(\mathbf{x}_t, t)-\nabla_{\mathbf{x}_t}\log p_{t}(\mathbf{x}_t\mid \mathbf{x}_0)\right\|^2\right]
\end{equation}
where $\lambda(t)$ specifies the loss weighting schedule and $p_{t}\left(\mathbf{x}_t \mid \mathbf{x}_0\right)$ denotes the Gaussian transition kernel of the forward process. During sampling, $s_\theta(\mathbf{x}, t)$ serves as a surrogate for the true score in the reverse-time dynamics, reducing the general SDE in Eq.~(\ref{eq:2}) to the deterministic gradient flow: 
\begin{equation}\label{eq:simplify}
\mathrm{d} \mathbf{x}= s_{\theta}(\mathbf{x}_t, t) \mathrm{d} t
\end{equation}

\section{Analysis of ODE and SDE }\label{sec:comparison}

\subsection{Trade-offs Between ODE and SDE Solvers}
\label{trade-off}

The choice between ODE and SDE solvers in DMs entails  trade-offs among sampling speed, quality, and error dynamics.
ODE solvers, characterized by deterministic trajectories, offer computational efficiency and stability through compatibility with compatibility with higher-order numerical methods, \textit{e.g}., iPNDM \cite{liupseudo,zhangfast}. 
 Such solvers reduce local discretization errors and achieve competitive sample quality with as few as 10–50 steps~\cite{lu2022dpm,karras2022elucidating}. 
 However, their deterministic nature limits their ability to correct errors from imperfect score function approximations, leading to performance plateaus as step count increases~\cite{xu2023restartsamplingimprovinggenerative}.
 Furthermore, the absence of stochasticity may suppress fine-grained variations, potentially reducing sample diversity compared to SDE-based methods~\cite{ho2020denoisingdiffusionprobabilisticmodels}.
 
 In contrast, SDE solvers leverage stochasticity to counteract accumulated discretization and gradient errors over time, enabling superior sample fidelity in high-step regimes~\cite{xu2023restartsamplingimprovinggenerative}. The injected noise further encourages exploration of the data manifold, improving diversity~\cite{ho2020denoisingdiffusionprobabilisticmodels}.  However, these benefits come at the cost of significantly larger step counts (typically 100–1{,}000) required to suppress errors that scale as $O(\delta^{3/2})$, compared to $O(\delta^{2})$ for ODEs~\cite{xu2023restartsamplingimprovinggenerative,Dalalyan_2019}. Moreover, SDE trajectories are highly sensitive to suboptimal noise schedules, particularly in low-step settings~\cite{nichol2021improveddenoisingdiffusionprobabilistic}.
  While reverse-time SDEs theoretically guarantee convergence to the true data distribution under ideal conditions \cite{ANDERSON1982313}, their computational cost often renders them impractical for real-time applications.
 
 Recent hybrid approaches, such as Restart sampling \cite{xu2023restartsamplingimprovinggenerative}, reconcile these trade-offs by alternating deterministic steps with stochastic resampling, leveraging ODE efficiency for coarse trajectory simulation while resetting errors via SDE-like noise injection. This strategy highlights the complementary strengths of both methods, positioning hybrid frameworks at the forefront of quality-speed Pareto frontiers in diffusion-based generation. 
However, Restart sampling still performs under high-step regimes ($>$50 steps). 

\subsection{Error Propagation in Deterministic and Stochastic Sampling}
The trade-offs discussed in~\Cref{trade-off} raise a key question: 

\begin{center}
\textit{Can SDE-based approaches achieve efficient sampling with substantially fewer steps?}
\end{center}

To answer this, we build on the theoretical frameworks of~\cite{xu2023restartsamplingimprovinggenerative,Dalalyan_2019} to analyze the total sampling error of ODE and SDE formulations under the Wasserstein-1 metric.
We begin with the discretized ODE system $\mathsf{ODE}_\theta$, governed by the learned drift field $s_\theta$, and examine its approximation behavior over the interval $\left[t, t+\Delta t\right] \subset [0, T]$.
Theorem~\ref{ode:error} formalizes this analysis and establishes an upper bound on the Wasserstein-1 distance between the generated and true data distributions (proof in Appendix~\ref{appsec:proof-the-1}).

\begin{theorem}(\textbf{ODE Error Bound}~\cite{xu2023restartsamplingimprovinggenerative})\label{ode:error}
Let $\Delta t>0$ denote the discretization step size. 
Over the interval $[t,t+\Delta t]$, the trajectory $\mathbf{x}_t=\mathsf{ODE}_\theta\left(\mathbf{x}_{t+\Delta t}, t+\Delta t \rightarrow t\right)$ is generated by the learned drift $s_\theta$, and the induced distribution is denoted by $p_t^{\mathsf{ODE}_\theta}$. We make the following assumptions:\\
\textbf{A1. Lipschitz and bounded drift:} $t s_\theta(\mathbf{x}, t)$ is $L_2$-Lipschitz in $\mathbf{x}$, $L_0$-Lipschitz in $t$ and uniformly bounded by $L_1$.\\
\textbf{A2:}  The learned drift satisfies a uniform supremum bound: 
$\sup _{\mathbf{x}, t}\left\|t s_\theta(\mathbf{x}, t)-t \nabla \log p_t(\mathbf{x})\right\| \leq \epsilon_t.$\\
\textbf{A3. Bounded trajectories:} $\left\|\mathbf{x}_t\right\| \leq B / 2$   for all $t \in\left[t, t+\Delta t\right]$.\\
The Wasserstein-1 distance between $p_t^{\mathsf{ODE}_\theta}$ and the true distribution $p_t$ satisfies:
\begin{equation}\label{eq:ode-error}
\underbrace{W_1\left(p_t^{\mathsf{ODE}_\theta}, p_t\right)}_{\text{total error}}
\leq \underbrace{B \cdot \mathsf{TV}\left(p_{t+\Delta t}^{\mathsf{ODE}_\theta}, p_{t+\Delta t}\right)}_{\text{\ding{192} gradient error bound}} + 
\underbrace{e^{L_2 \Delta t}\left(\Delta t(L_2 L_1 + L_0) + \epsilon_t\right) \Delta t}_{\text{\ding{193} discretization error bound}}
\end{equation}
where $\mathsf{TV}(\cdot, \cdot)$ denotes the total variation distance.
\end{theorem}

The bound in Eq.~\ref{eq:ode-error} consists of two term distinct interpretations.
The first term \ding{192} is the gradient error bound which reflects the discrepancy between the learned score function and the ground-truth one at the start time $t+\Delta t$. It also captures the propagation of errors accumulated from earlier time steps. 
The second term \ding{193} is the discretization error bound, which represents the newly introduced errors within the current interval $[t,t+\Delta t]$. 
Since the ODE process is deterministic, any discrepancy between the generated and true distributions at $t+\Delta t$ is directly carried forward to time $t$, without stochastic mechanisms to dissipate it.

Next, we introduce our \ours~update over the interval $[t, t+\Delta t]$, defined as:
\[
\mathbf{x}_t = 
\mathsf{AdaSDE}_\theta\!\left(\mathbf{x}_{t+\Delta t},\, t+\Delta t \!\rightarrow\! t,\, \gamma\right),
\]
which inserts a stochastic \textit{forward perturbation} followed by a deterministic \textit{backward process}.
\begin{align}
   &\mathbf{x}_{t+\Delta t }^{\gamma}
  = \mathbf{x}_{t + (1+\gamma)\Delta t}
  = \mathbf{x}_{t+\Delta t}
  + \varepsilon_{t+\Delta t \rightarrow\, t+(1+\gamma)\Delta t} \tag{Forward process}, \\
  &\mathbf{x}_{t}
  = \mathsf{ODE}_\theta\!\left(\mathbf{x}_{t+\Delta t }^{\gamma},\, t+(1+\gamma)\Delta t \rightarrow t\right) ,\tag{Backward process} 
\end{align}
where 
\[
\varepsilon_{t+\Delta t \rightarrow\, t+(1+\gamma)\Delta t}
~\sim~
\mathcal{N}\!\left(
\mathbf{0},
\left( (t+(1+\gamma)\Delta t)^2-(t+\Delta t)^2 \right)\mathbf{I}
\right).
\]
Here, $\gamma \in (0,1)$ is a tunable coefficient with its optimization deferred in \Cref{sec:method}. 
Different from deterministic ODE, \ours~introduces controlled noise injection to mitigate error accumulation. 
\Cref{sde:error} establishes an error bound between the generated and the true data distribution for our \ours~(proof in Appendix~\ref{appsec:proof-the-2}).

\begin{theorem}\label{sde:error}
    Under the same assumptions in \Cref{ode:error}. Let $p_{t}^{\mathsf{AdaSDE}_{\theta}}$ denote the distribution after \textup{$\ours$} update over the interval $[t,t+\Delta t]$. Then
    
\begin{align}
W_1\left(p_{t}^{\mathsf{AdaSDE}_{\theta}}, p_{t}\right) \leq  \underbrace{B \cdot(1-\lambda(\gamma)) \mathsf{TV}\left(p_{t +(1+\gamma )\Delta t }^{\mathsf{AdaSDE}}, p_{{t +(1+\gamma )\Delta t }}\right)}_{\text {gradient error bound}} \\
 +\underbrace{e^{(1+\gamma) L_2\Delta t}(1+\gamma)\left((1+\gamma)\Delta t\left(L_2 L_1+L_0\right)+\epsilon_t\right)\Delta t}_{\text{ discretization error bound}}
\end{align}
    where $\lambda(\gamma)=2\,Q\bigl(\dfrac{B}{2\,\sqrt{(t+(1+\gamma)\, \Delta t)^2 - t^2}}\bigr)$, $Q(r)=\mathrm{Pr}(a \geq r)$ for $a \sim \mathcal{N}(0,1)$.  
\end{theorem}
As shown in \Cref{sde:error}, the decoupled formulation tightens the Wasserstein-1 error bound through a reduced coefficient $B(1-\lambda(\gamma))$.
We next formalize this improvement by comparing the gradient-error terms of ODE and \ours~formulations in \Cref{thm:ODE-restart-compare}.
\begin{theorem}
\label{thm:ODE-restart-compare}
Under the same assumptions as in \Cref{ode:error} and \Cref{sde:error}, we denote:
\begin{align}
&\mathcal{E}_{\mathsf{grad}}^{\mathsf{ODE}}
= B \cdot
\mathsf{TV}\Bigl(p_{t+\Delta t}^{\mathsf{ODE}_{\theta}},\,p_{t+\Delta t}\Bigr), \tag{\text{ODE gradient error}}\\
&\mathcal{E}_{\mathsf{grad}}^{\mathsf{AdaSDE}}
= B\cdot\bigl(1-\lambda(\gamma)\bigr)\,\mathsf{TV}\Bigl(p_{t+(1+\gamma)\Delta t}^{\mathsf{AdaSDE}},\,p_{t+(1+\gamma)\Delta t}\Bigr). \tag{\text{SDE gradient error}}
\end{align}

Then we have $\mathcal{E}_\mathsf{grad}^{\mathsf{AdaSDE}}\leq 
\mathcal{E}_\mathsf{grad}^{\mathsf{ODE}}$, where the inequality is strict when $\gamma >0$.
\end{theorem}

\begin{proof}[Proof sketch] (full proof in Appendix~\ref{appsec:proof-thm-3})
For the ODE update, $\mathcal{E}_\mathsf{grad}^{\mathsf{ODE}}$ depends on the total-variation distance between the distributions at time $t+\Delta t$. For $\ours$ update, $\mathcal{E}_{\mathsf{AdaSDE}}$ includes a contraction factor $(1-\lambda(\gamma))$ and is evaluated at the higher noise level $t+(1+\gamma)\Delta t$.
Define the Gaussian kernel
\[
\phi_\gamma(z)
=(2\pi\sigma_\gamma^2)^{-d/2}\exp\!\left(-\frac{\|z\|^2}{2\sigma_\gamma^2}\right),\qquad 
\sigma_\gamma^2=\bigl(t+(1+\gamma)\Delta t\bigr)^2-\bigl(t+\Delta t\bigr)^2.
\]
The distributions after the noise injection satisfy
\[
p_{t+(1+\gamma)\Delta t}=p_{t+\Delta t}\ast\phi_\gamma,
\qquad
q_{t+(1+\gamma)\Delta t}=q_{t+\Delta t}\ast\phi_\gamma.
\]
By Lemma~\ref{TVnotdecrease} in Appendix, convolution with the same Gaussian kernel does not increase total variation distance:
\[
\mathsf{TV}\!\left(p_{t+\Delta t}\!\ast\!\phi_\gamma,\; q_{t+\Delta t}\!\ast\!\phi_\gamma\right)
\le
\mathsf{TV}\!\left(p_{t+\Delta t},\, q_{t+\Delta t}\right).
\]
Consequently,
\[
\mathcal{E}_\mathsf{grad}^{\mathsf{AdaSDE}}
\;\le\;
(1-\lambda(\gamma))\,\mathcal{E}_\mathsf{grad}^{\mathsf{ODE}},
\]
with a strictly smaller bound whenever $\gamma>0$.
\end{proof}

Although the gradient error term of \ours~enjoys a tighter bound through $B(1-\lambda(\gamma))$,
its discretization error grows rapidly under large time steps $(\Delta t)$ with noise schedules scaling as $\gamma(t)\!\propto\!\Delta t$.
 Specifically, the exponential growth factor $e^{(1+\gamma) L_2 \Delta t}$ combined with the quadratic $\Delta t$-dependence in $(1+\gamma)^2 \Delta t^2\left(L_2 L_1+L_0\right)$ creates error amplification that scales asymptotically as $O(\Delta t e^{C \Delta t})$ when $\gamma \sim O(\Delta t)$. This dominates the improved gradient error control, particularly during critical initial denoising steps where the product $(1+\gamma) \Delta t$ violates discretization stability conditions. This amplification offsets the benefit of gradient-error contraction,
causing total error accumulation along the trajectory and explaining the degraded few-step performance of SDE-based sampling in practice.

\subsection{Synthetic Validation}

\begin{figure*}[t]
    \centering
\includegraphics[width=\textwidth]{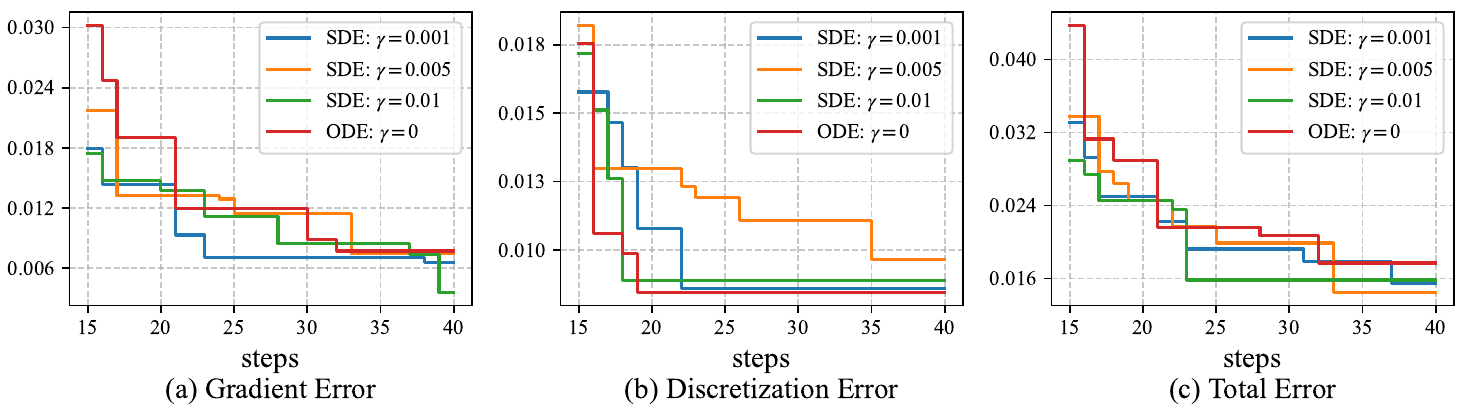}

    \caption{Gradient error, Discretization error and Total error on synthetic dataset across various steps (measured in 1-Wasserstein Distance). $\gamma=0$ indicates adding no stochasticity (ODE), $\gamma>0$ indicates SDE variants, figures are plotted in Pareto Frontier.}\label{fig:synthetic_experiments}
\end{figure*}
To verify the error-mitigation capability of stochastic updates in \ours, we conduct experiments on a 2D double-circle synthetic dataset, comparing the total, gradient, and discretization errors.

\begin{wrapfigure}[18]{r}{0.4\textwidth} 
\vspace{-6pt}                             
\centering
\includegraphics[width=0.35\textwidth]{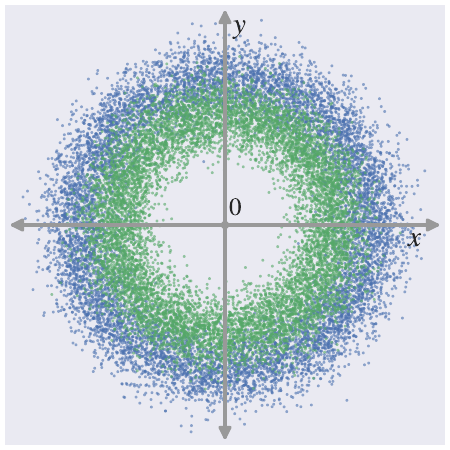}
\caption{Illustration of the 2D double-circles: two concentric rings with radii $0.8$ (outer, blue) and $0.6$ (inner, green). We uniformly sample $20{,}000$ points and add isotropic Gaussian noise ($\sigma=0.1$).}
\label{fig:toydataset}
\vspace{-6pt}                             
\end{wrapfigure}
\noindent\textbf{Setup.}  As illustrated in~\Cref{fig:toydataset}, we use a 2D double-circle dataset consisting of $20,000$ samples uniformly distributed along two concentric circles with radii of $0.8$ (outer) and $0.6$ (inner), each perturbed by Gaussian noise with a standard deviation of $0.1$.
We follow the training and sampling procedures of EDM~\cite{Karras2022edm} to model the data distribution, employing the second-order Heun method for SDE integration.
The stochastic coefficient~$\gamma$ is varied over $\{0, 0.001, 0.005, 0.01\}$, where $\gamma=0$ corresponds to the deterministic ODE sampler.

To quantify different types of errors, we measure the 2D Wasserstein-1 distance between corresponding distributions.
\textbf{The total error} is computed as the distance between the ground-truth data distribution and the generated distribution.
To estimate \textbf{gradient} and \textbf{discretization errors}, we first construct an intermediate regenerated distribution.
Specifically, given the dataset of $20,000$ samples, we perturb each point by Gaussian noise according to \(\mathbf{\mathbf{x}}_{t_{\text{mid}}} = \mathbf{\mathbf{x}}_0 + t_{\text{mid}}\sigma\), where \(t_{\text{mid}} = 0.8\) and perform one-third of a denoising step to obtain the regenerated samples.
The gradient error is defined as the distance between the regenerated distribution and the model-generated distribution at $T = 80.0$,
while the discretization error is defined as the distance between the regenerated distribution and the ground-truth distribution.

\noindent\textbf{Result.} The gradient error, discretization error, and total error over the steps range $t \in [15, 40]$ are illustrated in~\Cref{fig:synthetic_experiments}. It is observed that the discretization error of ODEs is less than that of SDE variants (in~\Cref{fig:synthetic_experiments}~(b)), corresponding to the derived result that the upper bound for ODE sampling error (stated in~\Cref{ode:error}) is less than that for SDEs (stated in~\Cref{sde:error}) by a multiplication factor. However, the gradient error (\textit{i.e}., error caused by network approximation) of SDEs ($\gamma > 0$) drops compared to ODE counterparts (in~\Cref{fig:synthetic_experiments}~(a)), validating the Wasserstein-1 distance bound in~\Cref{thm:ODE-restart-compare}. The stochastic step is effective in alleviating the gradient error made by network approximation. Consequently, as shown in~\Cref{fig:synthetic_experiments}~(c), the total error accumulated throughout the sampling process decreases due to the reduction of gradient error brought by stochasticity, confirming the effectiveness of our approach in improving sampling accuracy.
 Given the above theoretical analysis and synthetic validation on Wasserstein-1 distance, we present the following remark.

\begin{remark}
    Let \(\mathcal{E}_{\mathsf{total}}(N, \gamma)\) represent the accumulated sampling error for a discretization of \(N\) steps with parameter \(\gamma\). Then for \(\forall N\in \mathbb{Z}^+\),  \(\exists ~\gamma\in (0,1) \) such that:
    \begin{equation*}
        \mathcal{E}_{\mathsf{total}}(N, \gamma) \leq \mathcal{E}_{\mathsf{total}}(N, 0)
    \end{equation*}
\end{remark}

\section{Methodology}\label{sec:method}

Building on the above theoretical and empirical validation, we introduce \ours, a single-step SDE solver that  parameterizes the stochastic coefficient~$\gamma$ as learnable variable.
This design unleashes the potential of SDE-based solvers under low-NFE regimes.

\subsection{Sampling Trajectory Geometry}\label{subsec:sampling trajectory}
The  trajectories generated by Eq.~(\ref{eq:simplify}) exhibit low complexity geometric features with implicit connections to annealed mean displacement, as established in previous work \cite{chen2024trajectoryregularityodebaseddiffusion,chen2024geometricperspectivediffusionmodels}. Each sample initialized from the noise distribution progressively approaches the data manifold through smooth, quasi-linear trajectories characterized by monotonic likelihood improvement. In addition, under identical dataset and time schedule, all sampling trajectories demonstrate geometric consistency across different sampling methods. This geometric insight motivates a discrete-time distillation framework. By strategically inserting intermediate temporal steps within student trajectories, we construct high-fidelity reference trajectories. This enables process-supervised optimization that rigorously determines the governing $\gamma$ parameters for trajectory segments. Specifically, given a student time schedule $\mathcal{T}_{\text {stu }}=\left\{t_0, t_1, \ldots, t_N\right\}$ with $N$ steps, we insert $M$ intermediate steps between $t_n$ and $t_{n+1}$ (denoted as $\mathcal{T}_{\text {tea }}=\{t_0, t_0^{(1)}, \ldots, t_0^{(M)}, t_1, \ldots, t_N\}$ ) to generate refined teacher trajectories. Notably, our interpolation scheme employs a flexible strategy that allows for selecting different time schedules based on various solvers. This adaptability enhances the fidelity of teacher trajectories.
\begin{figure*}[t]
\centering
\includegraphics[width=\textwidth]{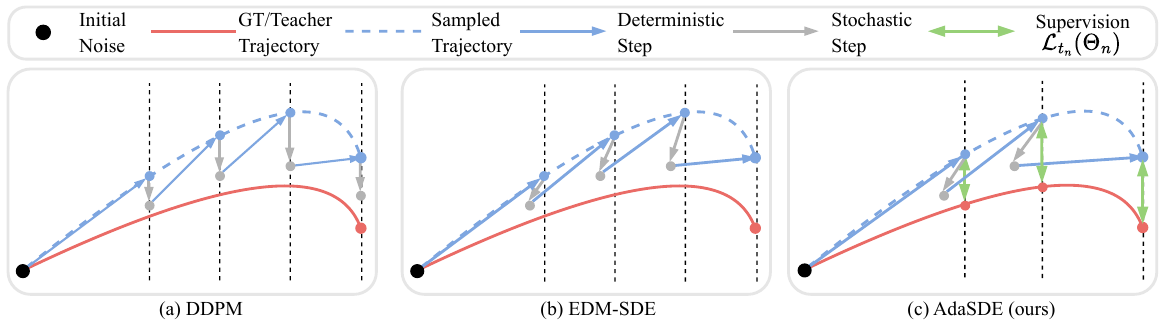}
\caption{The proposed \(\ours\) framework. {\(\ours\)} diverges from traditional heuristic noise injection methods used in DDPM~\cite{ho2020denoisingdiffusionprobabilisticmodels} and EDM-SDE~\cite{karras2022elucidating}. Instead, we use intermediary supervision from a teacher sampling path to learn and optimize the noise injection process.}
\label{fig:method}
\end{figure*}

\begin{figure}[t] 

\begin{minipage}[t]{0.48\textwidth}
\begin{algorithm}[H] 
\caption{Optimizing $\Theta_{1:N}$}\label{al_1}
\begin{algorithmic}[1]
\State \textbf{Given:} time schedules $\mathcal{T}_\mathsf{stu}$ and  $\mathcal{T}_\mathsf{tea}$
\State \textbf{Repeat until} convergence
\State \hspace{4mm} Sample $\rvx_{t_N}=\rvy_{t_N}\sim \mathcal{N}(\mathbf{0},t^2_N\mathbf{I})$ 
\State \hspace{4mm} \textbf{for $n = N$ \textbf{to} $1$ do} 
\State  \hspace{9mm} Sample $\bm{\epsilon}_{n} \sim \mathcal{N}(\mathbf{0}, \mathbf{I})$
    \State \hspace{9mm} ${\{\gamma,\xi,\lambda,\mu\}_{n} }\gets \Theta_{n}$
    \State \hspace{9mm} $\hat{t}_{n} \gets t_{n} + \gamma_{n} t_{n}$
    \State \hspace{9mm} $\mathbf{x}_{t_{n}} \gets \mathbf{x}_{t_{n}} + \sqrt{\hat{t}_{n}^2 - t_{n}^2}\boldsymbol{\epsilon}_n$ 
    \State \hspace{9mm} Compute $\rvx_{t_{n-1}}$ using~Eq.~(\ref{eq:finalx})
    \State \hspace{9mm} Update $\Theta_{n}$ via Eq.~(\ref{eq:loss})
    \State \hspace{4mm} \textbf{end for}

\end{algorithmic}
\end{algorithm}
\end{minipage}
\hfill 
\begin{minipage}[t]{0.48\textwidth}
\begin{algorithm}[H] 
\caption{$\ours$ sampling}\label{al_2}
\begin{algorithmic}[1]
\State \textbf{Given:} parameters $\Theta_{1:N}$, student time schedule $\mathcal{T}_\mathsf{stu}$
\State Initialize $\rvx_{t_N}\sim \mathcal{N}(\mathbf{0},t^2_N\mathbf{I})$
\For{$n = N$ \textbf{to} $1$}
    \State Sample $\bm{\epsilon}_{n} \sim \mathcal{N}(\mathbf{0}, \mathbf{I})$
    \State ${\{\gamma,\xi,\lambda,\mu\}_{n} }\gets \Theta_{n}$
    \State $\hat{t}_{n} \gets t_{n} + \gamma_{n} t_{n}$
    \State $\mathbf{x}_{t_{n}} \gets \mathbf{x}_{t_{n}} + \sqrt{\hat{t}_{n}^2 - t_{n}^2} \boldsymbol{\epsilon}_n$ 
    \State Compute $\rvx_{t_{n-1}}$ using Eq.~(\ref{eq:finalx})
\EndFor
\State \textbf{Return:} $\rvx_{t_0}$
\vspace{-0.8mm}
\end{algorithmic}
\end{algorithm}
\end{minipage}
\end{figure} 
\subsection{Fast SDE-based Sampling}
We extend the midpoint-based correction mechanisms Eq.~(\ref{eq:mean value}) from AMED-Solver \cite{zhou2024fast} to SDEs, proposing a sampling framework that strategically aligns stochastic perturbations with learned trajectory geometry. 
\begin{equation}\label{eq:mean value}
\mathbf{x}_{t_n} \approx \mathbf{x}_{t_{n+1}}+(t_n - t_{n+1}) \underbrace{s_\theta\left(\mathbf{x}_{\xi_n}, \xi_n\right)}_{\text{midpoint gradient}},{\xi_n} \in [t_{n+1},t_n] 
\end{equation}
The parameterization adopts the design from DPM-Solver's intermediate time step construction, formally defined as $\xi_n=\sqrt{t_n t_{n+1}}$. This square-root formulation guarantees smooth geometric interpolation between adjacent time steps in the noise scheduling process. Building on insights from~\cite{ning2024elucidatingexposurebiasdiffusion,li2024alleviatingexposurebiasdiffusion} showing network scaling mitigates input mismatches, we propose learnable parameters $\left\{\lambda_n, \mu_n\right\}$ to adaptively adjust both exposure bias and timestep scales.
The parameters $\Theta_n=\left\{\gamma_n, \xi_n, \lambda_n, \mu_n\right\}_{n=1}^N$ are optimized through our discrete-time distillation framework described in \Cref{subsec:sampling trajectory}. Consequently, Eq.~(\ref{eq:mean value}) can be reformulated in the following  form:
\begin{equation}\label{eq:finalx}
\mathbf{x}_{t_n} \approx \mathbf{x}_{t_{n+1}}+(1+\lambda_n)\left(t_n-t_{n+1}\right) s_\theta\left(\mathbf{x}_{\xi_n}, \xi_n+\mu_n\right)
\end{equation}
Let $\left\{\mathbf{y}_{t_n}\right\}_{n=1}^N$ denote the reference states of teacher trajectories. Starting from the identical initial noise $\mathbf{y}_{t_0}$, we generate student trajectories by optimizing the parameter sequence $\left\{\Theta_n\right\}_{n=1}^N$, resulting in student states $\left\{\mathbf{x}_{t_n}\right\}_{n=1}^N$ that align with the teacher trajectories under a predefined metric $d(\cdot, \cdot)$.
Crucially, since $\mathbf{x}_{t_n}$ depends on all preceding parameters $\left\{\Theta_n\right\}_{n=1}^N$ through the iterative sampling process, we implement stagewise optimization by minimizing the cumulative alignment loss at each timestep $t_n$ :
\begin{equation}\label{eq:loss}
\mathcal{L}_{t_n}({\Theta_n})= d\left(\mathbf{x}_{t_n}, \mathbf{y}_{t_n}\right)
\end{equation}
In each training loop, we perform backpropagation 
$N$ times. The comparison with existing SDE solvers are presented in \Cref{fig:method}. The complete training algorithm is detailed in \Cref{al_1}, while the inference procedure is outlined in \Cref{al_2}. AdaSDE serves as a plug-and-play module for existing solvers. To implement this, we train the AdaSDE predictor \Cref{al_1} by replacing the mean update in \Cref{eq:mean value} with the target solver's formulation.

\section{Experiments}

\label{sec:exp}

\subsection{Experiment Setup}
\noindent \textbf{Models and datasets.} 
We apply $\ours$ and DPM-Plugin to five pre-trained diffusion models across diverse domains. For pixel-space models, we include CIFAR10 (32 × 32) \cite{krizhevsky2009learning}, FFHQ (64 × 64) \cite{karras2019style}, and ImageNet (64 × 64) \cite{russakovsky2015imagenet}. For latent-space models, we evaluate LSUN Bedroom (256 × 256) \cite{yu2015lsun} with a guidance scale of 1.0. Additionally, we consider text-to-image high-resolution generation models, including Stable Diffusion v1.5 \cite{rombach2022high} at 512 × 512 resolution with a guidance scale of 7.5.
\makeatletter
\renewcommand\thesubtable{(\alph{subtable})}
\makeatother

\begin{table*}
\scriptsize
\captionsetup[subfloat]{labelformat=simple, labelsep=space}
\caption{
Image generation results across different datasets. 
(a) CIFAR10 \cite{liu2022flowstraightfastlearning} (unconditional), 
(b) FFHQ \cite{lin2014microsoft} (unconditional), 
(c) ImageNet \cite{russakovsky2015imagenet} (conditional), 
(d) LSUN Bedroom \cite{yu2015lsun} (unconditional). 
We compared AdaSDE-Solver and the training-required method AMED-Solver \cite{zhou2024fast}, as well as other training-free methods. AdaSDE achieves superior performance across all datasets.
}
\label{tab:main_results}
\begin{minipage}[t]{0.5\textwidth}

\subfloat[CIFAR10 $32 \times 32$ \cite{krizhevsky2009learning}]{
\begin{tabular}{lcccc}
\toprule
\multirow{2}{*}{\textbf{Method}} & \multicolumn{4}{c}{NFE} \\
\cmidrule{2-5} & 3 & 5 & 7 & 9 \\


\midrule
Multi-Step Solvers & & & &\\
\midrule
DPM-Solver++(3M)~\cite{lu2022dpm_plus} & 110.0 & 24.97 & 6.74 & 3.42 \\
UniPC~\cite{zhao2024unipc} & 109.6 & 23.98 & 5.83 & 3.21 \\
iPNDM~\cite{liupseudo,zhangfast} & 47.98 & 13.59 & 5.08 & 3.17 \\
\midrule
Single-Step Solvers & & & & \\
\midrule
DDIM~\cite{songdenoising} & 93.36 & 49.66 & 27.93 & 18.43 \\
Heun~\cite{karras2022elucidating} & 306.2 & 97.67 & 37.28 & 15.76 \\
DPM-Solver-2~\cite{lu2022dpm} & 153.6 & 43.27 & 16.69 & 8.65 \\
\rowcolor{lightCyan}
DPM-Plugin (ours) & 39.57 & 13.75 & 9.19  & 7.21  \\
AMED-Solver~\cite{zhou2024fast} & 18.49 & 7.59 & 4.36 & 3.67 \\
\rowcolor{lightCyan}
$\ours$ (ours) & \textbf{12.62}  & \textbf{4.18}  & \textbf{2.88} & \textbf{2.56}\\
\bottomrule
\end{tabular}
}

\vspace{1.5em} 
\subfloat[FFHQ $64 \times 64$ \cite{lin2014microsoft}]{
\begin{tabular}{lcccc}
\toprule
\multirow{2}{*}{\textbf{Method}} & \multicolumn{4}{c}{NFE} \\
\cmidrule{2-5} & 3 & 5 & 7 & 9 \\
\midrule
 Multi-Step Solvers & & & & \\
\midrule
DPM-Solver++(3M)~\cite{lu2022dpm_plus} & 86.45 & 22.51 & 8.44 & 4.77 \\
UniPC~\cite{zhao2024unipc} & 86.43 & 21.40 & 7.44 & 4.47 \\
iPNDM~\cite{liupseudo,zhangfast} & 45.98 & 17.17 & 7.79 & 4.58 \\
\midrule
Single-Step Solvers & & & & \\
\midrule
DDIM~\cite{songdenoising} & 78.21 & 43.93 & 28.86 & 21.01 \\
Heun~\cite{karras2022elucidating} &356.5 &116.7 &54.51 &28.86 \\
DPM-Solver-2~\cite{lu2022dpm} & 215.7 & 74.68 & 36.09 & 16.89 \\
\rowcolor{lightCyan}
DPM-Plugin (ours) & 66.31 & 20.80  & 14.51  & 10.48  \\
AMED-Solver~\cite{zhou2024fast}  & 47.31 & 14.80 & 8.82 & 6.31 \\
\rowcolor{lightCyan}
$\ours$ (ours) & \textbf{23.80}  & \textbf{8.05} & \textbf{5.11} & \textbf{4.19}\\
\bottomrule
\end{tabular}
}
\end{minipage}
\hfill
\begin{minipage}[t]{0.5\textwidth}
\centering

\subfloat[ImageNet $64 \times 64$ \cite{russakovsky2015imagenet}]{
\begin{tabular}{lcccc}
\toprule
\multirow{2}{*}{\textbf{Method}} & \multicolumn{4}{c}{NFE} \\
\cmidrule{2-5} & 3 & 5 & 7 & 9 \\
\midrule
 Multi-Step Solvers & & & & \\
\midrule
DPM-Solver++(3M)~\cite{lu2022dpm_plus} & 91.52 & 25.49 & 10.14 & 6.48 \\
UniPC~\cite{zhao2024unipc} & 91.38 & 24.36 & 9.57 & 6.34 \\
iPNDM~\cite{liupseudo,zhangfast} & 58.53 & 18.99 & 9.17 & 5.91 \\
\midrule
Single-Step Solvers & & & & \\
\midrule
DDIM~\cite{songdenoising} & 82.96 & 43.81 & 27.46 & 19.27 \\
Heun~\cite{karras2022elucidating} & 249.4 & 89.63 & 37.65 & 16.76 \\
DPM-Solver-2~\cite{lu2022dpm} & 140.2 & 59.47 & 22.02 & 11.31 \\
\rowcolor{lightCyan}
DPM-Plugin (ours) &108.9 & 17.03 &11.69  &8.06  \\
AMED-Solver~\cite{zhou2024fast} & 38.10 & 10.74 & 6.66 & 5.44 \\
\rowcolor{lightCyan}
$\ours$ (ours) & \textbf{18.51} & \textbf{6.90} & \textbf{5.26} & \textbf{4.59}\\
\bottomrule
\end{tabular}
}

\vspace{1.5em}

\subfloat[LSUN Bedroom $256 \times 256$ \cite{yu2015lsun}]{
\begin{tabular}{lcccc}
\toprule
\multirow{2}{*}{\textbf{Method}} & \multicolumn{4}{c}{NFE} \\
\cmidrule{2-5} & 3 & 5 & 7 & 9 \\
\midrule
Multi-Step Solvers & & & & \\
\midrule
DPM-Solver++(3M)~\cite{lu2022dpm_plus} & 111.9 & 23.15 & 8.87 & 6.45 \\
UniPC~\cite{zhao2024unipc} & 112.3 & 23.34 & 8.73 & 6.61 \\
iPNDM~\cite{liupseudo,zhangfast} & 80.99 & 26.65 & 13.80 & 8.38 \\
\midrule
Single-Step Solvers & & & & \\
\midrule
DDIM~\cite{songdenoising} & 86.13 & 34.34 & 19.50 & 13.26 \\
Heun~\cite{karras2022elucidating} & 291.5&175.7 & 78.66& 35.67 \\
DPM-Solver-2~\cite{lu2022dpm} & 227.3 & 47.22 & 23.21 & 13.80 \\
\rowcolor{lightCyan}
DPM-Plugin (ours) &97.13 &21.02  &13.68  &10.89  \\
AMED-Solver~\cite{zhou2024fast} & 58.21 & 13.20 & 7.10 & 5.65 \\
\rowcolor{lightCyan}
$\ours$ (ours) & \textbf{18.03} & \textbf{6.96} & \textbf{5.69} & \textbf{5.16} \\
\bottomrule
\end{tabular}
}
\end{minipage}

\end{table*}

\noindent\textbf{Solvers and time schedules.}
We compare $\ours$ against state-of-the-art single-step and multi-step ODE solvers. The single-step baselines include training-free methods—DDIM \cite{songdenoising}, EDM \cite{karras2022elucidating}, and DPM-Solver-2 \cite{lu2022dpm}, as well as the lightweight-tuning approach AMED-Solver \cite{zhou2024fast}. For multi-step methods, we evaluate DPM-Solver++ (3M) \cite{lu2022dpm_plus}, UniPC \cite{zhao2024unipc}, and iPNDM \cite{zhangfast,liupseudo}. To further demonstrate the effectiveness of our method, we also conduct a head-to-head comparison between DPM-Plugin and DPM-Solver-2 \cite{lu2022dpm}.

\begin{table}[t]
\centering
\begin{minipage}[t]{0.48\textwidth}
    \centering
    \caption{FID results on Stable Diffusion v1.5 \cite{rombach2022high} with a classifier-free guidance weight $w=7.5$.}
    \label{tab:sup_fid_stable_diff}
    \centering
\scriptsize
\label{tab:sd}
\begin{tabular}{lcccc}
  \toprule
  \multirow{2}{*}{\textbf{Method}} & \multicolumn{4}{c}{NFE} \\
  \cmidrule{2-5}
  & 4 & 6 & 8 & 10 \\
  \midrule
\textbf{MSCOCO 512×512} & & & &\\
\midrule
  DPM-Solver++(2M)~\cite{lu2022dpm_plus} & 21.33 & 15.99 & 14.84 & 14.58 \\
  AMED-Plugin~\cite{zhou2024fast} & \textbf{18.92} & {14.84} & {13.96} & {13.24} \\
    DPM-Solver-v3~\cite{zheng2023dpm} & - & 16.41 & 15.41 & 15.32 \\
  \rowcolor{lightCyan}
  $\ours$ (ours) & 30.89  & \textbf{13.99} & \textbf{13.39} & \textbf{12.68} \\
  \bottomrule
\end{tabular}
\end{minipage}
\hfill
\begin{minipage}[t]{0.48\textwidth}
    \centering
    \caption{Ablation study of time schedules on CIFAR-10 \cite{krizhevsky2009learning}.}
    \label{tab:schedule}
    \centering
\scriptsize
\begin{tabular}{lcccc}
\toprule
\multirow{2}{*}{\textbf{Time schedule}} & \multicolumn{4}{c}{NFE}\\
\cmidrule{2-5} & 3 & 5 & 7 & 9 \\
\midrule \textbf{CIFAR-10 32×32} & & & &\\
\midrule
Time Uniform~\cite{ho2020denoisingdiffusionprobabilisticmodels} & 12.62 & \textbf{4.18} & \textbf{2.88} & \textbf{2.56} \\
Time Polynomial  \cite{karras2022elucidating} & \textbf{11.61}& 10.05 & 5.14 & 3.35 \\
Time LogSNR \cite{lu2022dpm} & 23.38 & 10.42 & 7.96 & 4.84 \\
\bottomrule
\end{tabular}
\end{minipage}
\end{table}

To ensure an equitable and consistent comparison, our study faithfully adheres to the time scheduling strategies as recommended in the related work~\cite{karras2022elucidating,lu2022dpm_plus,zhao2024unipc}. Specifically, we implement the logarithmic signal-to-noise ratio (logSNR) scheduling for DPM-Solver\{-2, ++(3M)\} and UniPC algorithms. For other baseline algorithms, EDM time schedule with $\rho$ set to 7 has been employed. For AdaSDE and DPM-Plugin, we implement time-uniform schedule.

\noindent\textbf{Learned perceptual image patch similarity}
While some search-based frameworks employ LPIPS as their distance metric \cite{zhou2024simple}, we observed that using LPIPS during the intermediate steps of our method provided no significant performance gains and substantially increased training duration. Consequently, to balance efficiency and final quality, our approach utilizes Mean Squared Error (MSE) for optimizing intermediate steps, while applying the LPIPS metric in the final stage to enhance the overall training outcome.

\noindent\textbf{Training details.}
Our $\ours$ is assessed at low NFE settings  ($\text{NFE}\in \{3,5,7,9\}$) with AFS \cite{dockhorn2022geniehigherorderdenoisingdiffusion} implemented. Sample quality is gauged using the Fréchet Inception Distance (FID) \cite{ramesh2021zeroshottexttoimagegeneration} over 50k images. For Stable-Diffusion, We evaluate FID as \cite{ramesh2021zeroshottexttoimagegeneration}, using 30k samples from fixed prompts based on the MS-COCO \cite{lin2014microsoft} validation set. The random seed was fixed to 0 to ensure consistent reproducibility of the experimental results.


\subsection{Main Results}\label{sec:main_results}

In~\cref{tab:main_results}, we benchmark AdaSDE against single- and multi-step baseline solvers on CIFAR-10, FFHQ, ImageNet~64$\times$64, and LSUN Bedroom across varying NFE. We observe \emph{consistent and substantial} improvements in the low-step regime (3--9 NFE). For example, at NFE=9 we obtain FIDs of 4.59 (ImageNet) and 5.16 (LSUN Bedroom), while the second-best single-step baseline (AMED-Solver) reaches 5.44 and 5.65, respectively, indicating clear gains. In an even more challenging few-step setting (NFE=3 on LSUN Bedroom), AdaSDE achieves 18.03 FID, markedly outperforming AMED-Solver’s 58.21. On CIFAR-10, NFE=5 yields 4.18 FID (vs. AMED-Solver’s 7.59); on FFHQ, NFE=5 yields 8.05, substantially better than DPM-Plugin’s 20.80 and DPM-Solver-2’s 74.68. Overall, AdaSDE maintains—and often widens—its advantage as the number of steps decreases.

We further evaluate AdaSDE on Stable Diffusion v1.5 with classifier-free guidance set to $7.5$, reporting FID on the MS-COCO validation set (see~\cref{tab:sd}). At NFE=8/10, AdaSDE attains 13.39/12.68, surpassing DPM-Solver++(2M) at 14.84/14.58 and AMED-Plugin at 13.96/13.24, while remaining competitive with DPM-Solver-v3 across multiple step counts. These results indicate that our adaptive stochastic coefficient not only improves pixel-space diffusion models but also transfers robustly to high-resolution text-to-image generation in latent space. Additional quantitative results are provided in \Cref{fig:sup_grid_cifar10_2,fig:sup_grid_cifar10_3,fig:sup_grid_cifar10_4}.

\subsection{Ablation Studies}\label{sec:ablations}

\noindent\textbf{Effect of the stochastic coefficient.}
We quantify the contribution of the learned stochastic coefficient by comparing \ours\ with and without $\gamma_n$ on CIFAR-10, FFHQ, and Stable Diffusion~v1.5 (MS-COCO); see~\cref{tab:ablation_gamma_cifar,tab:ablation_gamma_mscoco}. Removing $\gamma_n$ consistently degrades FID, with the effect most pronounced in the few-step regime. On CIFAR-10, FID rises from 12.62 to 13.32 at NFE=3 and from 4.18 to 4.36 at NFE=5. On FFHQ~$64\times64$, we observe similar trends: FID increases from 23.80 to 25.85 at NFE=3 and from 8.04 to 8.11 at NFE=5. The benefit is especially clear on SD~v1.5 (MS-COCO~$512\times512$): when $\gamma_n$ is removed, FID rises from 30.89 to 37.23 at NFE=4 and from 13.79 to 16.34 at NFE=6, while the gap narrows as steps increase (12.68 with $\gamma_n$ versus 12.82 without at NFE=10). These results support that injecting learned stochasticity stabilizes few-step trajectories and mitigates error accumulation in low-NFE sampling.

\begin{table*}[t]
\scriptsize
\begin{minipage}[t]{0.48\textwidth} 
\centering
\caption{Ablation of {$\gamma_n$} on CIFAR-10 \cite{krizhevsky2009learning} and FFHQ \cite{lin2014microsoft}.} 
\begin{tabular}{lcccc}
\toprule
\multirow{2}{*}{\textbf{Training configuration}} & \multicolumn{4}{c}{NFE}\\
\cmidrule{2-5} & 3 & 5 & 7 & 9 \\
\midrule \textbf{CIFAR-10 32×32} & & & &\\
  \rowcolor{lightCyan}
\midrule $\ours$ & \textbf{12.62} & \textbf{4.18} & \textbf{2.88} & \textbf{2.56} \\
\quad w.o. $\gamma_n$ & 13.32 & 4.36 & 2.91 & 2.63 \\
\midrule \textbf{FFHQ 64×64} & & & &\\
  \rowcolor{lightCyan}
\midrule $\ours$ & \textbf{23.80} & \textbf{8.04} & \textbf{5.11} & \textbf{4.19} \\
\quad w.o. $\gamma_n$ & 25.85 & 8.11 & 5.12 & 4.27 \\

\bottomrule
\end{tabular}

\label{tab:ablation_gamma_cifar}
\end{minipage}\hfill 
\begin{minipage}[t]{0.48\textwidth} 
\centering
\caption{Ablation of {$\gamma_n$} on Stable Diffusion v1.5~\cite{rombach2022high}.} 
\begin{tabular}{lcccc}
\toprule
\multirow{2}{*}{\textbf{Training configuration}} & \multicolumn{4}{c}{NFE}\\
\cmidrule{2-5} & 4 & 6 & 8 & 10 \\
\midrule \textbf{MSCOCO 512×512} & & & &\\
  \rowcolor{lightCyan}
\midrule $\ours$ & \textbf{30.89} & \textbf{13.79} & \textbf{13.39} & \textbf{12.68} \\
\quad w.o. $\gamma_n$ & 37.23 & 16.34 & 14.18 & 12.82 \\
\bottomrule
\end{tabular}
\label{tab:ablation_gamma_mscoco}
\end{minipage}
\end{table*}

\noindent\textbf{Effect of time schedule.}
We further compare common time schedules on CIFAR-10—LogSNR, EDM (polynomial), and time-uniform—summarized in~\cref{tab:schedule}. The time-uniform schedule is the most reliable once NFE is at least 5, achieving FID scores of 4.18, 2.88, and 2.56 at NFE=5, 7, and 9, respectively, clearly outperforming the polynomial (10.05, 5.14, 3.35) and LogSNR (10.42, 7.96, 4.84) schedules. At the extreme NFE=3 setting, the polynomial schedule attains a marginally lower FID than the uniform schedule (11.61 versus 12.62), but its performance degrades rapidly as NFE increases. Overall, we adopt the time-uniform schedule as the default for few-step experiments due to its robustness across moderate step counts.

\section{Conclusion and Limitation}
\label{sec:conclusion}
\noindent \textbf{Conclusion.} In this work, we present AdaSDE, a novel framework using adaptive stochastic coefficient optimization to fundamentally address the efficiency-quality trade-off in diffusion sampling. It achieves new state-of-the-art results, such as a 4.18 FID on CIFAR-10 with only 5 NFE (a 1.8x improvement over prior SOTA). AdaSDE acts as a lightweight plugin, compatible with existing single-step solvers and requiring only 8-40 parameters for tuning, enabling practical deployment without full model retraining.

\noindent \textbf{Limitation.} When the step size is large and stronger stochastic injection is used (higher $\gamma$), local errors can amplify across steps and dominate the total sampling error, leading to instability. In practice, the admissible range of $\gamma$ is constrained by both the dataset and the step schedule, often necessitating conservative time discretization or $\gamma$ clipping.
 Our method’s per-step distribution resets and geometric alignment break the linear recurrence assumptions underlying multistep (e.g., iPNDM \cite{zhangfast,liupseudo}, UniPC \cite{zhao2024unipc}) and predictor–corrector frameworks. 

\bibliographystyle{unsrtnat}
\bibliography{ref}

\newpage
\appendix
{\LARGE\bf Appendix}
\appendix
\setcounter{theorem}{0}  


\section{Notation and Symbols for the Proof} 
This subsection provides a comprehensive list of notations and symbols specific to the theoretical proof. The definitions align with the conventions in stochastic calculus and diffusion model analysis.
We build on the notations of \cite{xu2023restartsamplingimprovinggenerative}. 
\subsection{Common Terms}
\begin{itemize}
    \item $\mathsf{ODE}_\theta(\cdot)$ : Approximate ODE trajectory using the learned score $s_\theta(\mathbf{x}, t)$.
    \item $p_t$ : True data distribution at noise level $t$.
    \item $p_t^{\mathsf{ODE}_\theta}$ : Distribution generated by simulating $\mathsf{ODE}_\theta$.
    \item $B$ : Norm upper bound for trajectories, satisfying $\forall t,\left\|\mathbf{x}_t\right\|<B / 2$.
    \item $\mathbf{x}_t\sim p_t$ : \(\mathbf{x}_t\) is sampled from distribution \(p_t\).
\end{itemize}
\subsection{AdaSDE Terms}
\begin{itemize}
    \item $\Delta t$ : ODE discretization step size.
    \item $\gamma$ : Hyperparameter controlling the noise injection ratio in the AdaSDE process.
    \item $\mathbf{x}_{t+\Delta t}^\gamma$ : AdaSDE forward process: $\mathbf{x}_{t+\Delta t}+\varepsilon_{t+\Delta t \rightarrow t+(1+\gamma) \Delta t}$.
    \item $\varepsilon$: Gaussian noise $\sim \mathcal{N}(0, I)$.
    \item $\mathbf{x}_t^\gamma$ : AdaSDE backward process: $\mathsf{ODE}_\theta\left(\mathbf{x}_{t+\Delta t}^\gamma, t+(1+\gamma) \Delta t \rightarrow t\right)$.
    \item $\mathsf{AdaSDE}_\theta(\mathbf{x}, \gamma)$ : Applies the AdaSDE operation with parameter $\gamma$ to state $\mathbf{x}$.
    \item $\bar{\mathbf{x}}_t$: The solution to
$
d \bar{\mathbf{x}}_t=-t s_\theta\left(\mathbf{x}_{t+\Delta t}, t+\Delta t\right) d t,
$

\end{itemize}
\subsection{Lipschitz and Error Bounds}
\begin{itemize}
    \item $L_0$ : Temporal Lipschitz constant:$\left\|t s_\theta(\mathbf{x}, t)-t s_\theta(\mathbf{x}, s)\right\| \leq L_0|t-s|$
    \item $L_1$ : Boundedness of the learned score: $\left\|t s_\theta(\mathbf{x}, t)\right\| \leq L_1$.
    \item $L_2$ : Spatial Lipschitz constant:$\left\|t s_\theta(\mathbf{x}, t)-t s_\theta(\mathbf{y}, t)\right\| \leq L_2\|\mathbf{x}-\mathbf{y}\|$
    \item $\epsilon_t$ : Score matching error:$\left\|t \nabla_\mathbf{x} \log p_t(\mathbf{x})-t s_\theta(\mathbf{x}, t)\right\|$
\end{itemize}

\subsection{Special Operators}
\begin{itemize}
    \item $\mathsf{ODE}\left(\mathbf{x}, t_1 \rightarrow t_2\right):$ Ground Truth backward ODE evolution under the exact score from $t_1$ to $t_2$.
    \item $\mathsf{ODE}_\theta\left(\mathbf{x}, t_1 \rightarrow t_2\right)$ : Approximate ODE evolution using the learned score $s_\theta$.
    \item $*$: Convolution operator between distributions, e.g., $P * R$ denotes the convolution of $P$ and $R$.
    \item $\leftarrow$ : Time-reversal marker, e.g., $\mathbf{x}_t^{\leftarrow}$.
\end{itemize}

\subsection{Key Process Terms}
\begin{itemize}
    \item $p_t^{\mathbf{x}, \gamma}$ : Distribution at noise level $t$ after applying the AdaSDE process starting from state $\mathbf{x}$.
    \item $p_t^{\mathsf{AdaSDE}_\theta}$ : Distribution generated by the AdaSDE algorithm.
    \item $\xi_x, \xi_y$ : i.i.d Gaussian noise: $\xi_x \sim \mathcal{N}\left(0, \sigma^2 I_d\right)$, $\xi_y \sim \mathcal{N}\left(0, \sigma^2 I_d\right)$.
\end{itemize}

\subsection{Error Dynamics}
\begin{itemize}
    \item $e(t):=\left\|\mathbf{x}_t^{\leftarrow}-\bar{\mathbf{x}}_t^{\leftarrow}\right\|:$ Error dynamics in the time-reversed coordinate system in \(t\).
    \item $\lambda(\gamma)$ : Noise merging probability:$2 Q\left(\frac{B}{2 \sqrt{(t+(1+\gamma) \Delta t)^2-t^2}}\right)$, where $Q(r)=\operatorname{Pr}(a \geq r)$ for $a \sim \mathcal{N}(0,1)$.
    \item $W_{1}(\cdot, \cdot)$ : Wasserstein-1 distance.
    \item $\mathsf{TV}$ $(\cdot, \cdot)$ : Total Variation (TV) distance.
\end{itemize}

Where $\varepsilon_{t+\Delta t \rightarrow t+(1+\gamma)\Delta t} \sim \mathcal{N}\left(\mathbf{0},\left((t+(1+\gamma)\Delta t)^2-(t+\Delta t)^2\right) \boldsymbol{I}\right)$. For the sake of simplifying symbolic representation and facilitating comprehension, in the following proof, we use $\mathsf{AdaSDE}_\theta(\mathbf{x}, \gamma)$ to denote $\mathbf{x}_{t}^{\gamma}$ in the above processes. In various theorems, we will refer to a function $Q(r): \mathbb{R}^{+} \rightarrow$ $[0,1 / 2)$, defined as the Gaussian tail probability $Q(r)=\operatorname{Pr}(a \geq r)$ for $a \sim \mathcal{N}(0,1)$.

\section{Proofs of Main Theoretical Results}
\label{Proofs of Main Theoretical Results}

\begin{lemma} [Upper Bound on ODE Discretization Error] \label{ODE Discretization Error}
\cite{xu2023restartsamplingimprovinggenerative}   
Let $\mathbf{x}_{t}=\mathsf{ODE}\left(\mathbf{x}_{t + \Delta t}, t + \Delta t \rightarrow t\right)$ denote the solution of the backward ODE under the exact score field,
and $\bar{\mathbf{x}}_{t}=\mathsf{ODE}_\theta\left(\bar{\mathbf{x}}_{t + \Delta t}, t + \Delta t \rightarrow t\right)$ denote the discretized ODE solution using the learned field $s_\theta$.
Assume $s_\theta$ satisfies:\\
1. Temporal Lipschitz Continuity:

$$
\left\|t s_\theta(\mathbf{x}, t)-t s_\theta(\mathbf{x}, s)\right\| \leq L_0|t-s| \quad \forall \mathbf{x}, t, s
$$

2. Boundedness:

$$
\left\|t s_\theta(\mathbf{x}, t)\right\| \leq L_1 \quad \forall \mathbf{x}, t
$$

3. Spatial Lipschitz Continuity:

$$
\left\|t s_\theta(\mathbf{x}, t)-t s_\theta(\mathbf{y}, t)\right\| \leq L_2\|\mathbf{x}-\mathbf{y}\| \quad \forall \mathbf{x}, \mathbf{y}, t
$$

Then the discretization error satisfies:
$$
\left\|\mathbf{x}_{t}-\bar{\mathbf{x}}_{t}\right\| \leq e^{L_2\Delta t}\left(\left\|\mathbf{x}_{t + \Delta t}-\bar{\mathbf{x}}_{t + \Delta t}\right\|+\left(\Delta t\left(L_2 L_1+L_0\right)+\epsilon_{t}\right) \Delta t\right)
$$
\end{lemma}
\begin{proof}
Step 1: Definition of Time-Reversed Processes\\
Introduce time-reversed variables $\mathbf{x}_t^{\leftarrow}$ and $\bar{\mathbf{x}}_t^{\leftarrow}$ governed by:
where $k$ is the integer satisfying $t \in[t',t' + \Delta t)$, corresponding to discrete timesteps.

Step 2: Error Dynamics\\
Define the error $e(t):=\left\|\mathbf{x}_t^{\leftarrow}-\bar{\mathbf{x}}_t^{\leftarrow}\right\|$. Its derivative satisfies:

$$
\frac{d}{d t} e(t) \leq\left\|t \nabla \log p_{t}\left(\mathbf{x}_t^{\leftarrow}\right)-t s_\theta\left(\bar{\mathbf{x}}_{t'+\Delta t}^{\leftarrow}, t'+\Delta t\right)\right\| .
$$

Decompose the right-hand side:

$$
\begin{aligned}
& \leq \underbrace{\left\|t \nabla \log p_{t}\left(\mathbf{x}_t^{\leftarrow}\right)-t s_\theta\left(\mathbf{x}_t^{\leftarrow},t\right)\right\|}_{\text {Approximation Error } \epsilon_t } \\
& +\underbrace{\left\|t s_\theta\left(\mathbf{x}_t^{\leftarrow},t\right)-t s_\theta\left(\bar{\mathbf{x}}_t^{\leftarrow},t\right)\right\|}_{L_2\left\|\mathbf{x}_t^{\leftarrow}-\bar{\mathbf{x}}_t^{\leftarrow}\right\|} \\
& +\underbrace{\left\|t s_\theta\left(\bar{\mathbf{x}}_t^{\leftarrow},t\right)-t s_\theta\left(\bar{\mathbf{x}}_{t'+\Delta t}^{\leftarrow}, t’+\Delta t\right)\right\|}_{\text {Temporal Discretization Error }} .
\end{aligned}
$$

Step 3: Temporal Discretization Error Bound\\
Further decompose the temporal discretization error:

$$
\begin{aligned}
& \leq\left\|t s_\theta\left(\bar{\mathbf{x}}_t^{\leftarrow},t\right)-t s_\theta\left(\bar{\mathbf{x}}_{t'+\Delta t}^{\leftarrow},t\right)\right\|+\left\|t s_\theta\left(\bar{\mathbf{x}}_{t'+\Delta t}^{\leftarrow},t\right)-t s_\theta\left(\bar{\mathbf{x}}_{t' + \Delta t}^{\leftarrow}, t' + \Delta t\right)\right\| \\
& \leq L_0|t'+\Delta t-t'|+L_2\|\bar{\mathbf{x}}_{t}^{\leftarrow}-\bar{\mathbf{x}}_{t'+\Delta_t}^{\leftarrow}\| \quad \text { (Lipschitz continuity) } \\
& \leq L_0 \Delta t+L_2\left(\left\|\bar{\mathbf{x}}_{t}^{\leftarrow}-\bar{\mathbf{x}}_{t'+\Delta t}^{\leftarrow}\right\|\right) .
\end{aligned}
$$

Using the boundedness condition $\left\|d \bar{\mathbf{x}}_t^{\leftarrow} / d t\right\| \leq L_1$, we have:

$$
\left\|\bar{\mathbf{x}}_t^{\leftarrow}-\bar{\mathbf{x}}_{t'+\Delta t}^{\leftarrow}\right\| \leq \int_{t}^{t'+\Delta t} \left\|d \bar{\mathbf{x}}_s^{\leftarrow}\right\| d s \leq L_1 \Delta t
$$

Step 4: Composite Differential Inequality\\
Combining all terms, the error dynamics satisfy:

$$
\frac{d}{d t} e(t) \leq L_2 e(t)+\left(\epsilon_{t}+L_0 \Delta t+L_2 L_1 \Delta t \right)
$$

Step 5: Gronwall's Inequality Application\\
Integrate over $t \in\left[t,t+\Delta t\right]$ and apply Gronwall's inequality:
$$
e\left(t\right) \leq e^{L_2 \Delta t}\left(e\left(t + \Delta t\right)+\left(\epsilon_{t}+\Delta t \left(L_0+L_2 L_1\right)\right) \Delta t\right)
$$

\end{proof}

\begin{lemma}[TV Distance Between Gaussian Perturbations] \label{TV Distance Between Gaussian Perturbations}
Let $\xi_x \sim \mathcal{N}(0, \sigma^2 I_d)$ and $\xi_y \sim \mathcal{N}(0, \sigma^2 I_d)$ be independent noise vectors. For $\mathbf{x}' = \mathbf{x} + \xi_x$ and $\mathbf{y}' = \mathbf{y} + \xi_y$, their total variation distance satisfies:
$$
\mathsf{TV}(\mathbf{x}', \mathbf{y}') = 1 - 2Q\left(\frac{\|\mathbf{x} - \mathbf{y}\|}{2\sigma}\right)
$$
where $Q(r) = \Pr_{a\sim\mathcal{N}(0,1)}(a \geq r)$.
\end{lemma}

\begin{proof}
Let $\delta = \mathbf{x}-\mathbf{y}$. The TV distance is:

\begin{align*}
\mathsf{TV}(\mathbf{x}', \mathbf{y}') 
&= \frac{1}{2}\int_{\mathbb{R}^d}|\mathcal{N}(\mathbf{z};\mathbf{x},\sigma^2 I_d) - \mathcal{N}(\mathbf{z};\mathbf{y},\sigma^2 I_d)|dz \\
&= \frac{1}{2}\int_{\mathbb{R}^d}\left|\mathcal{N}(\mathbf{z}-\delta;\mathbf{0},\sigma^2 I_d) - \mathcal{N}(\mathbf{z};\mathbf{0},\sigma^2 I_d)\right|dz
\end{align*}

Through orthogonal transformation $U$ aligning $\delta$ with the first axis:
$$
U\delta = (\|\delta\|, 0,...,0)^\top
$$
By rotational invariance of Gaussians:
$$
\mathsf{TV}(\mathbf{x}', \mathbf{y}') = \mathsf{TV}\left(\mathcal{N}(\|\delta\|,\sigma^2), \mathcal{N}(0,\sigma^2)\right)
$$

For 1D Gaussians $\mathcal{N}(\mu,\sigma^2)$ and $\mathcal{N}(0,\sigma^2)$:
\begin{align*}
\mathsf{TV} &= \frac{1}{2}\int_{-\infty}^\infty\left|\phi\left(\frac{\mathbf{z}-\mu}{\sigma}\right) - \phi\left(\frac{\mathbf{z}}{\sigma}\right)\right|dz \\
&= \Phi\left(-\frac{\mu}{2\sigma}\right) - \Phi\left(\frac{\mu}{2\sigma}\right) \quad \text{(By symmetry)} \\
&= 1 - 2\Phi\left(\frac{\mu}{2\sigma}\right) = 2Q\left(\frac{\mu}{2\sigma}\right)
\end{align*}
where $\mu = \|\mathbf{x} - \mathbf{y}\|$. 
then:
$$
\mathsf{TV}(\mathbf{x}', \mathbf{y}') = 1 - 2Q\left(\frac{\|\delta\|}{2\sigma}\right)= 1 - 2Q\left(\frac{\|\mathbf{x} - \mathbf{y}\|}{2\sigma}\right).
$$
\end{proof}

\begin{lemma} \label{TVAdaSDE}
Let $p_t^{\mathbf{x},\gamma}$ and $p_t^{\mathbf{y},\gamma}$ denote the densities of $\mathbf{x}_t^\gamma$ and $\mathbf{y}_t^\gamma$ respectively. After applying AdaSDE with noise injection from $t$ to $t+(1+\gamma)\Delta t$ followed by backward ODE evolution, we have:
$$
\mathsf{TV}\left(p_t^{\mathbf{x},\gamma}, p_t^{\mathbf{y},\gamma}\right) \leq (1-\lambda(\gamma)) \mathsf{TV}\left(p_t^\mathbf{x}, p_t^\mathbf{y}\right)
$$
where $\lambda(\gamma) = 2Q\left(\dfrac{B}{2\sqrt{(t+(1+\gamma)\Delta t)^2 - t^2}}\right)$.
\end{lemma} 

\begin{proof}
Consider states $\mathbf{x}_t$ and $\mathbf{y}_t$ at noise level $t$ with $\|\mathbf{x}_t - \mathbf{y}_t\| \leq B$. The AdaSDE process first perturbs both states to noise level $t+(1+\gamma)\Delta t$ through Gaussian noise injection:
\begin{align*}
\mathbf{x}_{t+(1+\gamma)\Delta t} &= \mathbf{x}_t + \xi_x,\ \xi_x \sim \mathcal{N}(0, [(t+(1+\gamma)\Delta t)^2 - t^2]I) \\
\mathbf{y}_{t+(1+\gamma)\Delta t} &= \mathbf{y}_t + \xi_y,\ \xi_y \sim \mathcal{N}(0, [(t+(1+\gamma)\Delta t)^2 - t^2]I)
\end{align*}

We construct a coupling between the noise injections: when $\mathbf{x}_t = \mathbf{y}_t$, set $\xi_x = \xi_y$; otherwise use reflection coupling. By Lemma \ref{TV Distance Between Gaussian Perturbations}, the merging probability satisfies:
\begin{align*}
\lambda(\gamma) = 2Q\left(\frac{\|\mathbf{x}_t - \mathbf{y}_t\|}{2\sigma_t(\gamma)}\right) \geq 2Q\left(\frac{B}{2\sigma_t(\gamma)}\right) \quad (\text{since } \|\mathbf{x}_t - \mathbf{y}_t\| \leq B)
\end{align*}
where $Q(r) = \Pr_{a\sim\mathcal{N}(0,1)}(a \geq r)$.

This implies:
$$
\mathbb{P}(\mathbf{x}_{t+(1+\gamma)\Delta t} \neq \mathbf{y}_{t+(1+\gamma)\Delta t} \mid \mathbf{x}_t \neq \mathbf{y}_t) \leq 1 - \lambda(\gamma)
$$
where $\lambda(\gamma)$ quantifies the minimum merging probability between the Gaussian perturbations.

The subsequent backward ODE evolution preserves this coupling relationship because both trajectories are driven by the same learned score $s_\theta$. Therefore:
$$
\mathbb{P}(\mathbf{x}_t^\gamma \neq \mathbf{y}_t^\gamma) \leq (1-\lambda(\gamma))\mathbb{P}(\mathbf{x}_t \neq \mathbf{y}_t)
$$

Through the coupling characterization of total variation distance, we conclude:
$$
\mathsf{TV}(p_t^{\mathbf{x},\gamma}, p_t^{\mathbf{y},\gamma}) \leq (1-\lambda(\gamma))\mathsf{TV}(p_t^\mathbf{x}, p_t^\mathbf{y}) \qedhere
$$
\end{proof}

\begin{lemma}[AdaSDE Error Propagation]
\label{lem:adasde-error}
Let $\mathbf{x}_{t + \Delta t}\in\mathbb{R}^d$ be an initial point. Define exact and approximate ODE solutions:
$$
\begin{aligned}
\mathbf{x}_{t} &= \mathsf{ODE}\bigl(\mathbf{x}_{t + (1+\gamma)\Delta t},\,t + (1+\gamma)\Delta t\!\to t\bigr), \\
\hat{\mathbf{x}}_{t} &= \mathsf{ODE}_\theta\bigl(\hat{\mathbf{x}}_{t + (1+\gamma)\Delta t},\,t + (1+\gamma)\Delta t\!\to t\bigr).
\end{aligned}
$$
Under AdaSDE with noise injection $t+\Delta t \to t+(1+\gamma)\Delta t$ and $\|\mathbf{x}_{t}-\hat{\mathbf{x}}_{t}\|\le B$, there exists a coupling such that:
\begin{align*}
\bigl\|\,
\mathbf{x}_{t+(1+\gamma)\Delta t}
-\,
\hat{\mathbf{x}}_{t+(1+\gamma)\Delta t}
\bigr\|
&\leq 
e^{L_2(1+\gamma)\Delta t}(1+\gamma)\left[\Delta t(L_2L_1 + L_0) + \epsilon_t\right]\Delta t,
\end{align*}
where $L_0,L_1,L_2,\epsilon_{t}$ are the Lipschitz/boundedness/approximation constants for $s_\theta$ and discretization errors.
\end{lemma}

\begin{proof}
\textbf
By Lemma~\ref{ODE-error} (ODE Discretization Error), the local truncation error satisfies:
\begin{align*}
\|\mathbf{x}_t - \hat{\mathbf{x}}_t\| 
&\leq e^{L_2(1+\gamma)\Delta t}\Big[\|\mathbf{x}_{t+(1+\gamma)\Delta t} - \hat{\mathbf{x}}_{t+(1+\gamma)\Delta t}\| \\
&\quad + \underbrace{\big((1+\gamma)\Delta t(L_2L_1+L_0) + \epsilon_t\big)(1+\gamma)\Delta t}_{\text{Local discretization error}}\Big].
\end{align*}

Applying AdaSDE's noise injection with variance $\sigma^2 = (t+(1+\gamma)\Delta t)^2 - t^2$, Lemma~\ref{TV Distance Between Gaussian Perturbations} gives:
\begin{align*}
\mathbb{E}\|\mathbf{x}_{t+(1+\gamma)\Delta t} - \hat{\mathbf{x}}_{t+(1+\gamma)\Delta t}\| 
&\leq (1-\lambda(\gamma))\|\mathbf{x}_t - \hat{\mathbf{x}}_t\| ,
\end{align*}
where the merging probability $\lambda(\gamma)=2\,Q\bigl(\dfrac{B}{2\,\sqrt{(t+(1+\gamma)\, \Delta t)^2 - t^2}}\bigr)$ dominates the coupling effectiveness.

Multiplying by $(1-\lambda(\gamma))$ from partial revert and adding the local ODE approximation error leads to the stated bound:
\begin{align*}
\bigl\|\,
\mathbf{x}_{t+(1+\gamma \Delta t)}
-\,
\hat{\mathbf{x}}_{t+(1+\gamma \Delta t)}
\bigr\|
\;\; \le&\;\; 
\bigl(1-\lambda(\gamma)\bigr)\,\bigl\|\,\mathbf{x}_{t}-\hat{\mathbf{x}}_{t}\bigr\|
\; \\ &+\;
e^{L_2\,(1+\gamma) \Delta t}\,(1+\gamma)\bigl[(1+\gamma)\Delta t(\,L_2L_1 + L_0)\;+\;\epsilon_{t}\bigr]\, \Delta t\\
& = e^{L_2\,(1+\gamma) \Delta t}\,(1+\gamma)\bigl[\Delta t(\,L_2L_1 + L_0)\;+\;\epsilon_{t}\bigr]\, \Delta t
\end{align*}

\end{proof}

\begin{lemma}[Connection of Wasserstein-1 distance and Norm]
\label{lemma:wasserstein_l1}
Let $ p_1 $ and $ p_2 $ be two probability distributions over a space $ \mathcal{X} \subseteq \mathbb{R}^d $, and let $ \Gamma(p_1, p_2) $ denote the set of all joint distributions with marginals $ p_1 $ and $ p_2 $. The Wasserstein-1 distance between $ p_1 $ and $ p_2 $ satisfies:

$$
W_1(p_1, p_2) = \inf_{\psi \in \Gamma(p_1, p_2)} \mathbb{E}_{(\mathbf{x}_1, \mathbf{x}_2) \sim \psi} \left[ \| \mathbf{x}_1 - \mathbf{x}_2 \| \right],
$$

where $ \| \cdot \|_1 $ is the L1 norm. Furthermore, for independent samples $ \mathbf{x}_1 \sim p_1 $ and $ \mathbf{x}_2 \sim p_2 $, we have:

$$
W_1(p_1, p_2) \leq \mathbb{E} \left[ \| \mathbf{x}_1 - \mathbf{x}_2 \| \right],
$$

with equality if and only if the coupling $ \psi $ is optimal.
\end{lemma}

\begin{lemma}\label{TVnotdecrease}
     $\mathsf{TV}(P * R, Q * R) \leq \mathsf{TV}(P, Q)$ for independent distributions $P, Q$, and $R$.The inequality $\mathsf{TV}(P * R, Q * R)=\mathsf{TV}(P, Q)$ holds if and only if $R$ is a degenerate distribution.
\end{lemma}

\begin{proof}
    1. Total Variation Distance Definition

The total variation distance between two distributions $P$ and $Q$ is defined as:
$$
\mathsf{TV}(P, Q)=\frac{1}{2} \int_{-\infty}^{\infty}|p(\mathbf{x})-q(\mathbf{x})| d \mathbf{x}
$$
where $p(\mathbf{x})$ and $q(\mathbf{x})$ are the probability density functions of $P$ and $Q$, respectively.

2. Convolution Definition

The convolution of two distributions $P$ and $R$ is defined as:
$$
(P * R)(\mathbf{x})=\int_{-\infty}^{\infty} p(\mathbf{x}-\mathbf{y}) r(\mathbf{y}) d \mathbf{y}
$$

Similarly, for $Q$ and $R$ :
$$
(Q * R)(\mathbf{x})=\int_{-\infty}^{\infty} q(\mathbf{x}-\mathbf{y}) r(\mathbf{y}) d \mathbf{y}
$$

3. TV Distance for Convolved Distributions

We want to compute $\mathsf{TV}(P * R, Q * R)$, which is:
\begin{align*}
\mathsf{TV}(P * R, Q * R) 
&= \frac{1}{2} \int_{-\infty}^\infty \left| (P * R)(\mathbf{x}) - (Q * R)(\mathbf{x}) \right| dx \\
&= \frac{1}{2} \int_{-\infty}^\infty \left| \int_{-\infty}^\infty (p(\mathbf{x}-\mathbf{y}) - q(\mathbf{x}-\mathbf{y})) r(\mathbf{y}) dy \right| dx
\end{align*}

Applying triangle inequality, we obtain:
$$
\mathsf{TV}(P * R, Q * R) \leq \frac{1}{2} \int_{-\infty}^{\infty}\left(\int_{-\infty}^{\infty}|p(\mathbf{x}-\mathbf{y})-q(\mathbf{x}-\mathbf{y})| r(\mathbf{y}) d \mathbf{y}\right) d \mathbf{x}
$$
Using Fubini's theorem, we can swap the order of integration:
$$
\mathsf{TV}(P * R, Q * R) \leq \frac{1}{2} \int_{-\infty}^{\infty}\left(\int_{-\infty}^{\infty}|p(\mathbf{x}-\mathbf{y})-q(\mathbf{x}-\mathbf{y})| d \mathbf{x}\right) r(\mathbf{y}) d \mathbf{y}
$$
For fixed $\mathbf{y}$, the inner integral is:
$$
\int_{-\infty}^{\infty}|p(\mathbf{x}-\mathbf{y})-q(\mathbf{x}-\mathbf{y})| d \mathbf{x}=\int_{-\infty}^{\infty}|p(\mathbf{x})-q(\mathbf{x})| d \mathbf{x}
$$
Thus, we obtain:

$$
\mathsf{TV}(P * R, Q * R) \leq \frac{1}{2} \int_{-\infty}^{\infty}\left(\int_{-\infty}^{\infty}|p(\mathbf{x})-q(\mathbf{x})| d \mathbf{x}\right) r(\mathbf{y}) d \mathbf{y}
$$
$$
\mathsf{TV}(P * R, Q * R) \leq \mathsf{TV}(P, Q)
$$
The inequality $\mathsf{TV}(P * R, Q * R)=\mathsf{TV}(P, Q)$ holds if and only if $R$ is a degenerate distribution.

\end{proof}




\subsection{Proof of Theorem~\ref{ODE-error}}\label{appsec:proof-the-1}
\begin{theorem} \label{ODE-error}
Let $t+\Delta t$ be the initial noise level. Let $\mathbf{x}_{t}=\mathsf{ODE}_\theta\left(\mathbf{x}_{t+\Delta t}, t+\Delta t \rightarrow t\right)$ and $p_t^{\mathsf{ODE}_\theta}$ denote the distribution induced by simulating the ODE with learned drift $s_\theta$. Assume:\\
1. The learned drift $t s_\theta(\mathbf{x}, t)$ is $L_2$-Lipschitz in $\mathbf{x}$, bounded by $L_1$, and $L_0$-Lipschitz in $t$.\\
2. The approximation error $\left\|t s_\theta(\mathbf{x}, t)-t \nabla \log p_t(\mathbf{x})\right\| \leq \epsilon_{t}$.\\
3. All trajectories are bounded by $B / 2$.\\
Then, the Wasserstein-1 distance between the generated distribution $p_{t}^{\mathsf{ODE}_\theta}$ and the true distribution $p_{t}$ is bounded by:\\
$$
\begin{aligned}
W_1\left(p_{t}^{\mathsf{ODE}_\theta}, p_{t}\right) \leq  B \cdot \mathsf{TV}\left(p_{t+\Delta t}^{\mathsf{ODE}_\theta}, p_{t+\Delta t}\right)  +e^{L_2 \Delta t} \cdot\left(\Delta t\left(L_2 L_1+L_0\right)+\epsilon_{t}\right) \Delta t
\end{aligned}
$$
where $\Delta t$ is the step size
\end{theorem}

\begin{proof}
    Let $\hat{\mathbf{x}}_{t}=\mathsf{ODE}_\theta\left(\mathbf{x}_{t+\Delta t}, t+\Delta t \rightarrow t\right)$ with the corresponding distribution \(\hat{p}_t\) and $\mathbf{x}_{t}=\mathsf{ODE}\left(\mathbf{x}_{t+\Delta t}, t+\Delta t\rightarrow t\right)$ (simulated under the true score). 
    The proof bounds $W_1\left(p_{t}^{\mathsf{ODE}_\theta}, p_{t}\right)$ via triangular inequality:
\begin{equation}
    W_1\left(p_{t}^{\mathsf{ODE}_\theta}, p_{t}\right) \leq W_1\left(p_{t}^{\mathsf{ODE}_\theta}, \hat{p}_{t}\right) +  W_1\left(\hat{p}_t, p_{t}\right)
\end{equation}
Then we can bound two terms seperately.

1. gradient error: By bounded-diameter inequality,

$$
 W_1\left(p_{t}^{\mathsf{ODE}_\theta}, \hat{p}_{t}\right) \leq B \cdot \mathsf{TV}\left(p_{t+\Delta t}^{\mathsf{ODE}_\theta}, p_{t+\Delta t}\right)
$$

2. discretization error: Using Lemma \ref{ODE Discretization Error} (discretization bound), given \(\mathbf{x}_t \sim p_t, \hat{\mathbf{x}}_t \sim \hat{p}_t\)
$$
 \left\|\hat{\mathbf{x}}_{t}-\mathbf{x}_{t}\right\| \leq e^{L_2\Delta t} \cdot\left(\Delta t\left(L_2 L_1+L_0\right)+\epsilon_{t}\right)\Delta t
$$

where the exponential factor arises from Gronwall's inequality applied to the Lipschitz drift. According to Lemma \ref{lemma:wasserstein_l1}, we can
combine terms via triangular inequality:
$$
W_1\left(p_{t}^{\mathsf{ODE}_\theta}, p_t\right) \leq \underbrace{B \cdot \mathsf{TV}\left(p_{t+\Delta t}^{\mathsf{ODE}_\theta}, p_{t+\Delta t}\right)}_{\text {gradient error }}+\underbrace{e^{L_2\Delta t} \cdot\left(\Delta t\left(L_2 L_1+L_0\right)+\epsilon_{t}\right)\Delta t}_{\text {discretization error }}
$$
\end{proof}

\subsection{Proof of Theorem~\ref{AdaSDE-error}}\label{appsec:proof-the-2}

\begin{theorem}[AdaSDE Error Decomposition] \label{AdaSDE-error}
Consider the same setting as Theorem \ref{ODE-error}. Let $p_{t}^{\mathsf{AdaSDE}_{\theta}}$ denote the distribution after AdaSDE iteration. Then
$$
\begin{aligned}
W_1\left(p_{t}^{\mathsf{AdaSDE}_{\theta}}, p_{t}\right) \leq & \underbrace{B \cdot(1-\lambda(\gamma)) \mathsf{TV}\left(p_{t+(1+\gamma)\Delta t}^{\mathsf{AdaSDE}}, p_{t+(1+\gamma)\Delta t}\right)}_{\text{gradient error}} \\
& +\underbrace{e^{(1+\gamma) L_2\Delta t}(1+\gamma)\left((1+\gamma)\Delta t\left(L_2 L_1+L_0\right)+\epsilon_t\right)\Delta t}_{\text{discretization error}}
\end{aligned}
$$
where $\lambda(\gamma)=2Q\left(\dfrac{B}{2\sqrt{(t+(1+\gamma)\Delta t)^2 - t^2}}\right)$.
\end{theorem}

\begin{proof}
Let $\mathbf{x}_{t+(1+\gamma)\Delta t} \sim p_{t+(1+\gamma)\Delta t}$ and $\hat{\mathbf{x}}_{t+(1+\gamma)\Delta t} \sim p_{t+(1+\gamma)\Delta t}^{\mathsf{AdaSDE}}$. denote exact and generated distributions respectively. And $\bar{\mathbf{x}}_{t+(1+\gamma)\Delta t} \sim p_{t+(1+\gamma)\Delta t}^{\mathrm{\theta}}$. The proof contains three key components:

By Lemma \ref{TVAdaSDE}, the AdaSDE process contracts the TV distance:

\begin{align*}
\|\bar{\mathbf{x}}_t- \hat{\mathbf{x}}_t\|
&\leq (1-\lambda(\gamma))\|\bar{\mathbf{x}}_{t+(1+\gamma)\Delta t}- \hat{\mathbf{x}}_{t+(1+\gamma)\Delta t}\| \\
&= (1-\lambda(\gamma))\|\bar{\mathbf{x}}_{t+(1+\gamma)\Delta t}-\mathbf{x}_{t+(1+\gamma)\Delta t}\|
\end{align*}

Since \(\bar{\mathbf{x}}_t\sim p^{\theta}_t\) and \(\hat{\mathbf{x}}_t\sim p_{ t}^{\mathsf{AdaSDE}_{\theta}}\), we obtain:
\begin{align*}
\mathsf{TV}\left(\bar{p}_t, p_{ t}^{\mathsf{AdaSDE}_{\theta}}\right) 
&\leq (1-\lambda(\gamma))\mathsf{TV}\left(\bar{p}_{t+(1+\gamma)\Delta t}, \hat{p}_{t+(1+\gamma)\Delta t}\right) \\
&= (1-\lambda(\gamma))\mathsf{TV}\left(\bar{p}_{t+(1+\gamma)\Delta t}, p_{t+(1+\gamma)\Delta t}\right)
\end{align*}

 Using the bounded trajectory assumption $\|\mathbf{x}\| \leq B/2$, we convert TV to Wasserstein-1:
$$
W_1\left(\bar{p}_t, p_{ t}^{\mathsf{AdaSDE}_{\theta}}\right)  \leq B  \cdot \mathsf{TV}\left(\bar{p}_t, p_{ t}^{\mathsf{AdaSDE}_{\theta}}\right)  \leq B(1-\lambda(\gamma))\mathsf{TV}\left(\bar{p}_{t+(1+\gamma)\Delta t}, p_{t+(1+\gamma)\Delta t}\right)
$$

 From Lemma \ref{TVAdaSDE}, the local ODE error satisfies:
$$
\|\mathbf{x}_t^\gamma - \bar{\mathbf{x}}_t^\gamma\| \leq e^{(1+\gamma)L_2\Delta t}(1+\gamma)\left[(1+\gamma)\Delta t(L_2L_1 + L_0) + \epsilon_t\right]\Delta t
$$

According to Lemma \ref{lemma:wasserstein_l1} and  Apply triangle inequality to Wasserstein distances:
\begin{align*}
W_1\left(p_t^{\mathsf{AdaSDE}_\theta}, p_t\right) 
&\leq W_1\left(\bar{p}_t, p_{ t}^{\mathsf{AdaSDE_{\theta}}}\right) + W_1\left(\bar{p}_t, p_{t}\right) \\
&\leq B(1-\lambda(\gamma))\mathsf{TV}\left(p_{t+(1+\gamma)\Delta t}^{\mathsf{AdaSDE}}, p_{t+(1+\gamma)\Delta t}\right) \\
&\quad + e^{(1+\gamma)L_2\Delta t}(1+\gamma)\left[(1+\gamma)\Delta t(L_2L_1 + L_0) + \epsilon_t\right]\Delta t
\end{align*}
This completes the error decomposition. \qedhere
\end{proof}

\subsection{Proof of Theorem~\ref{thm:ODE-AdaSDE-compare}}\label{appsec:proof-thm-3}

\begin{theorem}[TV comparison: AdaSDE vs.\ ODE]\label{thm:ODE-AdaSDE-compare}
Assume the same conditions as in Theorem~\ref{ODE-error} and Theorem~\ref{AdaSDE-error}, 
and in particular that there exists a compact $K\subset\mathbb{R}^d$ with $diam(K)\le B$ such that the relevant one-step distributions are supported in $K$.
Define
\[
\begin{aligned}
\text{\emph{(i) ODE gradient:}}\qquad
& \mathcal{E}^{\mathsf{ODE}}_{\mathsf{grad}}
  := B \cdot
  \mathsf{TV}\!\Bigl(p_{t+\Delta t}^{\mathsf{ODE}_{\theta}},\,p_{t+\Delta t}\Bigr),\\
\text{\emph{(ii) AdaSDE gradient:}}\qquad
& \mathcal{E}^{\mathsf{AdaSDE}}_{\mathsf{grad}}
  := B\,\bigl(1-\lambda(\gamma)\bigr)\,
  \mathsf{TV}\!\Bigl(p_{t+(1+\gamma)\Delta t}^{\mathsf{AdaSDE}},\,p_{t+(1+\gamma)\Delta t}\Bigr).
\end{aligned}
\]

where $\lambda(\gamma)=2\,Q\!\Bigl(\dfrac{B}{2\,\sqrt{(t+(1+\gamma)\, \Delta t)^2 - t^2}}\Bigr)\in(0,1)$ and $B>0$ is the diameter bound.
Then
\[
\mathcal{E}^{\mathsf{AdaSDE}}_{\mathsf{grad}}
\;\le\;
\mathcal{E}^{\mathsf{ODE}}_{\mathsf{grad}}.
\]
\end{theorem}

\begin{proof}
By Theorem~\ref{ODE-error},
\[
\mathcal{E}^{\mathsf{ODE}}_{\mathsf{grad}}
=
B \cdot \mathsf{TV}\!\Bigl(
p_{t+\Delta t}^{\mathsf{ODE}_{\theta}},\,p_{t+\Delta t}
\Bigr).
\]
By Theorem~\ref{AdaSDE-error},
\[
\mathcal{E}^{\mathsf{AdaSDE}}_{\mathsf{grad}}
=
B\,\bigl(1-\lambda(\gamma)\bigr)\,
\mathsf{TV}\!\Bigl(
p_{t+(1+\gamma)\Delta t}^{\mathsf{AdaSDE}},\,
p_{t+(1+\gamma)\Delta t}
\Bigr).
\]
From $t+\Delta t$ to $t+(1+\gamma)\Delta t$, AdaSDE injects Gaussian noise (a common Markov kernel) into both branches. 
By Lemma~\ref{TVnotdecrease} (convolution/pushforward is nonexpansive in TV),
\[
\mathsf{TV}\!\Bigl(
p_{t+(1+\gamma)\Delta t}^{\mathsf{AdaSDE}},\,
p_{t+(1+\gamma)\Delta t}
\Bigr)
\;\le\;
\mathsf{TV}\!\Bigl(
p_{t+\Delta t}^{\mathsf{ODE}_{\theta}},\,
p_{t+\Delta t}
\Bigr).
\]
Since $0<(1-\lambda(\gamma))<1$, we get
\[
\mathcal{E}^{\mathsf{AdaSDE}}_{\mathsf{grad}}
=
B\,\bigl(1-\lambda(\gamma)\bigr)\,
\mathsf{TV}\!\Bigl(
p_{t+(1+\gamma)\Delta t}^{\mathsf{AdaSDE}},\,
p_{t+(1+\gamma)\Delta t}
\Bigr)
\;\le\;
B\cdot \mathsf{TV}\!\Bigl(
p_{t+\Delta t}^{\mathsf{ODE}_{\theta}},\,
p_{t+\Delta t}
\Bigr)
=
\mathcal{E}^{\mathsf{ODE}}_{\mathsf{grad}}.
\]
\end{proof}

\begin{remark}[When the inequality is strict]
If $\gamma>0$, the Gaussian kernel is nondegenerate, and 
$\mathsf{TV}\!\bigl(p_{t+\Delta t}^{\mathsf{ODE}_{\theta}},\,p_{t+\Delta t}\bigr)>0$
(equivalently, the two pre-smoothing distributions are not a.e.\ equal and admit $L^1$ densities), then
\[
\mathsf{TV}\!\Bigl(
p_{t+(1+\gamma)\Delta t}^{\mathsf{AdaSDE}},\,
p_{t+(1+\gamma)\Delta t}
\Bigr)
\;<\;
\mathsf{TV}\!\Bigl(
p_{t+\Delta t}^{\mathsf{ODE}_{\theta}},\,
p_{t+\Delta t}
\Bigr),
\]
and hence $\mathcal{E}^{\mathsf{AdaSDE}}_{\mathsf{grad}}<\mathcal{E}^{\mathsf{ODE}}_{\mathsf{grad}}$.
\end{remark}

\section{More on AdaSDE}
\subsection{Experiment details.}

\noindent\textbf{Experiment detail in main result}

Since AdaSDE has fewer than 40 parameters, its training incurs minimal computational cost. We train $\Theta$ for 10K images, which only takes 5-10 minutes on CIFAR10 with a single 4090 GPU and about 20 minutes on LSUN Bedroom with four 4090 GPUs. For generating reference teacher trajectories, we use DPM-Solver-2 with M=3.  For tuning across all datasets, we employed a learning rate of 0.2 along with a cosine learning rate schedule (coslr). The random seed was fixed to 0 to ensure consistent reproducibility of the experimental results. To ensure the robustness of our experimental results, we conducted ten independent runs for each NFE (Number of Function Evaluations) setting on the CIFAR10 dataset. Across these runs, the FID (Fréchet Inception Distance) scores consistently varied by no more than 0.1.

\subsection{Time uniform scheme}

\cite{ho2020denoisingdiffusionprobabilisticmodels} proposes a discretization scheme for diffusion sampling given the starting $\sigma_{\max }$, end time $\sigma_{\min }$ and $\epsilon_s$. Denote the number of steps as $N$, then the \textit{time uniform} discretization scheme is:
\begin{align*}
\sigma(t) &= \left(e^{0.5\,\beta_d\, t^2 + \beta_{\min}\, t}-1\right)^{0.5} \\[2pt]
\sigma^{-1}(\sigma) &= 
\frac{\sqrt{\beta_{\min}^2 + 2\,\beta_d \,\ln\!\left(\sigma^2+1\right)} - \beta_{\min}}{\beta_d} \\[4pt]
\beta_d &= \frac{2\left(\ln\!\left(\sigma_{\min}^2+1\right)/\epsilon_s - \ln\!\left(\sigma_{\max}^2+1\right)\right)}{\epsilon_s - 1} \\
\beta_{\min} &= \ln\!\left(\sigma_{\max}^2+1\right) - 0.5\,\beta_d \\[4pt]
t_{\text{temp}} &= \left(1 + \frac{i}{N-1}\left(\epsilon_s^{1/\rho}-1\right)\right)^{\rho} \\[2pt]
t_i &= \sigma\!\left(t_{\text{temp}}\right)
\end{align*}

We set  $\sigma_{\max }=80.0$, $\sigma_{\min}=0.002$, $\rho=1$ and $\epsilon_s = 10^{-3}$ across all datasets in our experiments.

\subsection{Supplementary experimental results}
\sisetup{
  group-digits = false
}

\newcommand{\fidhead}{\textbf{FID}~\textcolor{red}{\(\downarrow\)}}
\newcommand{\cliphead}{\textbf{CLIP (\%)}~\textcolor{green!60!black}{\(\uparrow\)}}

\begin{table}[!htbp]
\centering
\caption{Evaluation on MSCOCO 512$\times$512 (Flux.1-dev).}
\label{tab:coco-flux}
\setlength{\tabcolsep}{8pt}
\renewcommand{\arraystretch}{1.2}
\begin{tabular}{l c l S[table-format=2.2] S[table-format=2.2]}
\toprule
\textbf{Model} & \textbf{NFE} & \textbf{Sampler/Method} &
{\fidhead} & {\cliphead} \\
\midrule
\multirow{6}{*}{\textbf{Flux.1-dev 512$\times$512}}
  & \multirow{2}{*}{6} & DPM-Solver-2 & 54.09 & 28.49 \\
  &                     & \ours        & 35.32 & 29.94 \\
\cmidrule(lr){2-5}
  & \multirow{2}{*}{8} & DPM-Solver-2 & 30.17 & 29.75 \\
  &                     & \ours        & 26.51 & 30.51 \\
\cmidrule(lr){2-5}
  & \multirow{2}{*}{10} & DPM-Solver-2 & 26.32 & 30.32 \\
  &                     & \ours        & 23.54 & 30.77 \\
\bottomrule
\end{tabular}
\end{table}
\begin{figure}[ht]
\centering
\includegraphics[width=0.9\textwidth]{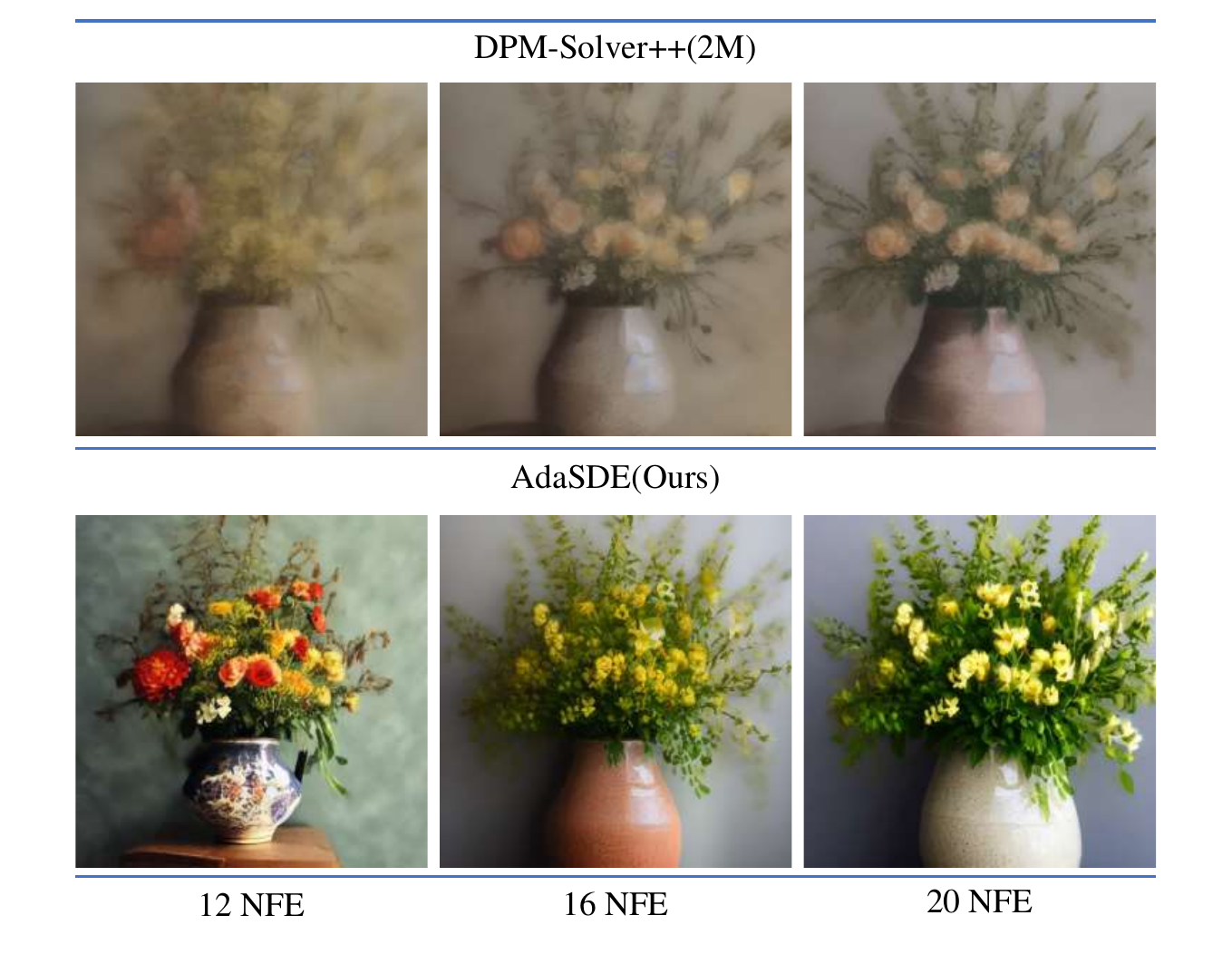}
\caption{Comparison of image synthesis quality under identical NFE constraints using AdaSDE (ours) and DPM-Solver++ (2M). Both methods generate images with Stable Diffusion v1.5~\cite{rombach2022high} and classifier-free guidance (scale = 7.5) for the prompt \textit{``A photo of some flowers in a ceramic vase"}.}
\label{fig:flowers}
\end{figure}
\begin{table}[]
    \centering
    \small
    \setlength{\tabcolsep}{4pt}
    \renewcommand{\arraystretch}{1.2}
        \caption{Unconditional generation results on CIFAR10 32×32.}\label{table:cifar10_result}
    \begin{tabular}{lccrrrrrrrrr}
        \toprule
        \multirow{2}{*}{\textbf{Method}} & \multirow{2}{*}{\textbf{AFS}}  & \multicolumn{8}{c}{\textbf{NFE}} \\
        \cmidrule(lr){4-11}
         &  &  & 3 & 4 & 5 & 6 & 7 & 8 & 9 & 10 \\
         \midrule
        \multirow{1}{*}{DPM-Solver-v3} & $\times$ & & - & - & 15.10 & 11.39 & - & 8.96 & - & 8.27 \\
        \midrule
        \multirow{2}{*}{UniPC} & $\times$ & & 109.6 & 45.20 & 23.98 & 11.14 & 5.83 & 3.99 & 3.21 & 2.89 \\
         & $\checkmark$ & & 54.36 & 20.55 & 9.01 & 5.75 & 4.11 & 3.26 & 2.93 & 2.65 \\
         \midrule
        \multirow{2}{*}{DPM-Solver++(3M)} & $\times$ & & 110.0 & 46.52 & 24.97 & 11.99 & 6.74 & 4.54 & 3.42 & 3.00 \\
         & $\checkmark$ & & 55.74 & 22.40 & 9.94 & 5.97 & 4.29 & 3.37 & 2.99 & 2.71 \\
         \midrule
        \multirow{2}{*}{iPNDM} & $\times$ &  & 47.98 & 24.82 & 13.59 & 7.05 & 5.08 & 3.69 & 3.17 & 2.77 \\
         & $\checkmark$ &  & 24.54 & 13.92 & 7.76 & 5.07 & 4.04 & 3.22 & 2.83 & \textbf{2.56} \\
        \midrule
        \multirow{2}{*}{DDIM} & $\times$ &   & 93.36 & 66.76 & 49.66 & 35.62 & 27.93 & 22.32 & 18.43 & 15.69 \\
         & $\checkmark$ &   & 67.26 & 49.96 & 35.78 & 28.00 & 22.37 & 18.48 & 15.69 & 13.47 \\
         \midrule
        \multirow{2}{*}{DPM-Solver-2} & $\times$ &   & - & 205.41 & - & 45.32 & - & 12.93 & - & 10.65 \\
         & $\checkmark$ &   & 227.32 & - & 47.22 & - & 13.68 & - & 10.89 \\
      \midrule
         \multirow{2}{*}{AMED-Solver}& $\times$ &  & - & 17.18 & - & 7.04 & - & 5.56 & - & 4.14 \\
         & $\checkmark$ &  & 18.49 & - & 7.59 & - & 4.36 & - & 3.67 & - \\
        \midrule
         \multirow{2}{*}{AdaSDE (ours)} & $\times$ &  & - & \textbf{10.16} & - & \textbf{4.67}& - & \textbf{3.18} & - & 2.65 \\
         & $\checkmark$ &  & \textbf{12.62} & - & \textbf{4.18} & - & \textbf{2.88} & - & \textbf{2.56} & - \\
        \bottomrule
    \end{tabular}
\end{table}
\begin{table}[]
    \centering
    \small
    \setlength{\tabcolsep}{4pt}
    \renewcommand{\arraystretch}{1.2}
        \caption{Unconditional generation results on ImageNet 64×64.}\label{table:imagenet_result}
    \begin{tabular}{lccrrrrrrrrr}
        \toprule
        \multirow{2}{*}{\textbf{Method}} & \multirow{2}{*}{\textbf{AFS}}  & \multicolumn{8}{c}{\textbf{NFE}} \\
        \cmidrule(lr){4-11}
         &  &  & 3 & 4 & 5 & 6 & 7 & 8 & 9 & 10 \\
        \midrule
        \multirow{2}{*}{UniPC} & $\times$ & & 91.38 & 55.63 & 54.36 & 14.30 & 9.57 & 7.52 & 6.34 & 5.53 \\
         & $\checkmark$ & & 64.54 & 29.59 & 16.17 & 11.03 & 8.51 & 6.98 & 6.04 & 5.26 \\
         \midrule
        \multirow{2}{*}{DPM-Solver++(3M)} & $\times$ & & 91.52 & 56.34 & 25.49 & 15.06 & 10.14 & 7.84 & 6.48 & 5.67 \\
         & $\checkmark$ & & 65.20 & 30.56 & 16.87 & 11.38 & 8.68 & 7.12 & 6.25 & 5.58 \\
         \midrule
        \multirow{2}{*}{iPNDM} & $\times$ &  & 58.53 & 33.79 & 18.99 & 12.92 & 9.17 & 7.20 & 5.91 & 5.11 \\
         & $\checkmark$ &  & 34.81 & 21.31 & 15.53 & 10.27 & 8.64 & 6.60 & 5.64 &4.97 \\
        \midrule
        \multirow{2}{*}{DDIM} & $\times$ &   & 82.96 & 58.43 & 43.81 & 34.03 & 27.46 & 22.59 & 19.27& 16.72 \\
         & $\checkmark$ &   & 62.42 & 46.06 & 35.48 & 28.50 & 23.31 & 19.82 & 17.14 & 15.02 \\
         \midrule
        \multirow{2}{*}{DPM-Solver-2} & $\times$ &   & - & 140.20 & - & 59.47 & - & 22.02 & - & 11.31 \\
         & $\checkmark$ &   & 163.21 & - & 62.32 & - & 23.68 & - & 11.89 \\
      \midrule
         \multirow{2}{*}{AMED-Solver}& $\times$ &  & - & 32.69 & - & 10.63 & - & 7.71 & - & 6.06 \\
         & $\checkmark$ &  & 38.10 & - & 10.74 & - &  6.66 & - & 5.44 & - \\
        \midrule
         \multirow{2}{*}{AdaSDE (ours)} & $\times$ &  & - & \textbf{18.53} & - & \textbf{7.01}& - & \textbf{5.36} & - & \textbf{4.63} \\
         & $\checkmark$ &  & \textbf{18.51} & - & \textbf{6.90} & - & \textbf{5.26} & - & \textbf{4.59} & - \\
        \bottomrule
    \end{tabular}
\end{table}


\newcommand{\nfea}{\text{5}}
\newcommand{\nfeb}{\text{9}}

\newpage
\begin{figure*}
  \centering
  \setlength{\abovecaptionskip}{-0.03cm}
  \begin{subfigure}[b]{0.48\linewidth}
      \setlength{\abovecaptionskip}{0.03cm}
      \includegraphics[width=\linewidth]{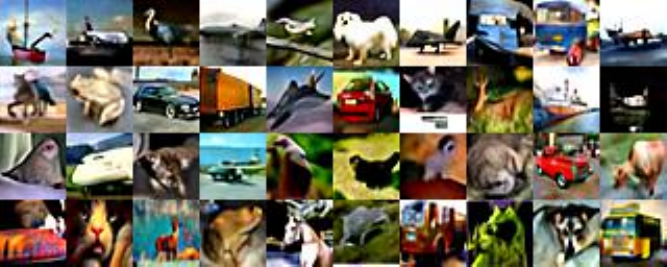}
      \caption{DPM-Solver-2. NFE=\(\nfea\), FID = 43.27}
  \end{subfigure}
  \begin{subfigure}[b]{0.48\linewidth}
        \setlength{\abovecaptionskip}{0.03cm}
      \includegraphics[width=\linewidth]{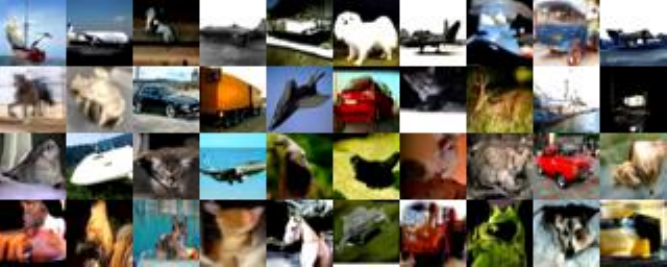}
      \caption{DPM-Solver-2. NFE=\(\nfeb\), FID = 8.65}
  \end{subfigure}
  \begin{subfigure}[b]{0.48\linewidth}
        \setlength{\abovecaptionskip}{0.03cm}
      \includegraphics[width=\linewidth]{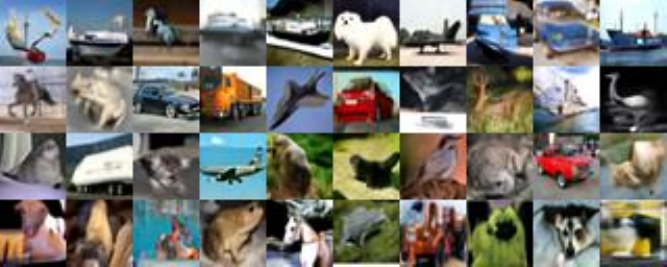}
      \caption{\(\ours\). NFE=\(\nfea\), FID = 4.18}
  \end{subfigure}
  \begin{subfigure}[b]{0.48\linewidth}
        \setlength{\abovecaptionskip}{0.03cm}
      \includegraphics[width=\linewidth]{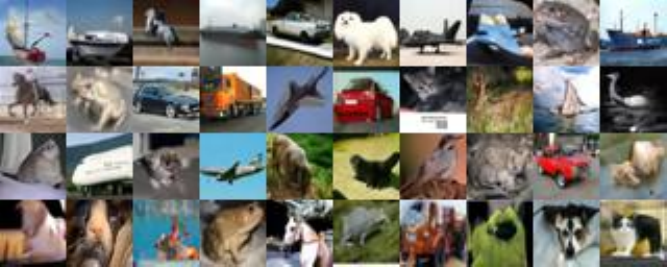}
      \caption{\(\ours\). NFE=\(\nfeb\), FID = 2.56}
  \end{subfigure}
  \setlength{\belowcaptionskip}{-0.35cm}
  \caption{Qualitative result on CIFAR10 32$\times$32 (\(\nfea\) and \(\nfeb\) NFEs)}
  
  \label{fig:sup_grid_cifar10_2}
\end{figure*}

\begin{figure*}
  \centering
  \setlength{\abovecaptionskip}{-0.03cm}
  \begin{subfigure}[b]{0.48\linewidth}
        \setlength{\abovecaptionskip}{0.03cm}
      \includegraphics[width=\linewidth]{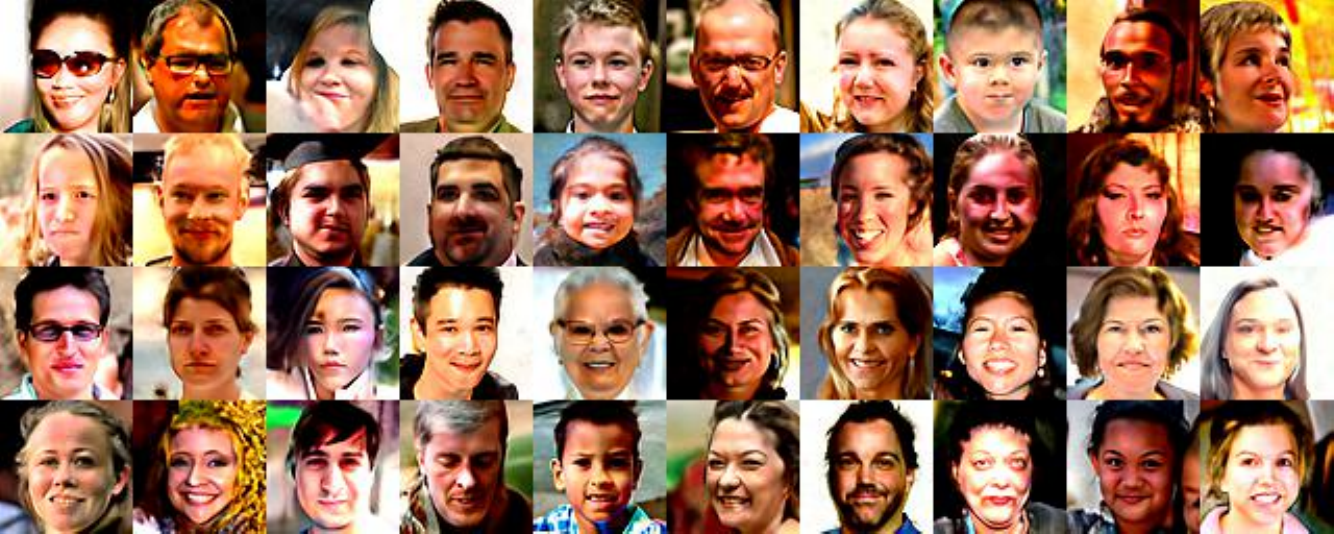}
      \caption{DPM-Solver-2. NFE=\(\nfea\), FID = 74.68}
  \end{subfigure}
  \begin{subfigure}[b]{0.48\linewidth}
        \setlength{\abovecaptionskip}{0.03cm}
      \includegraphics[width=\linewidth]{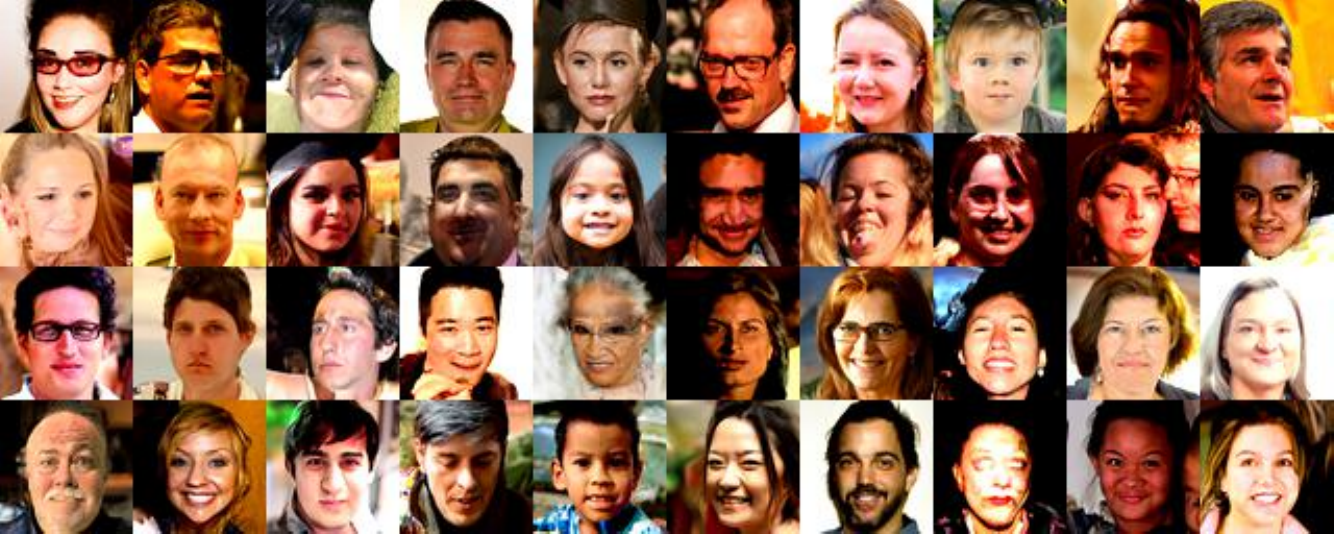}
      \caption{DPM-Solver-2. NFE=\(\nfeb\), FID = 16.89}
  \end{subfigure}
  \begin{subfigure}[b]{0.48\linewidth}
        \setlength{\abovecaptionskip}{0.03cm}
      \includegraphics[width=\linewidth]{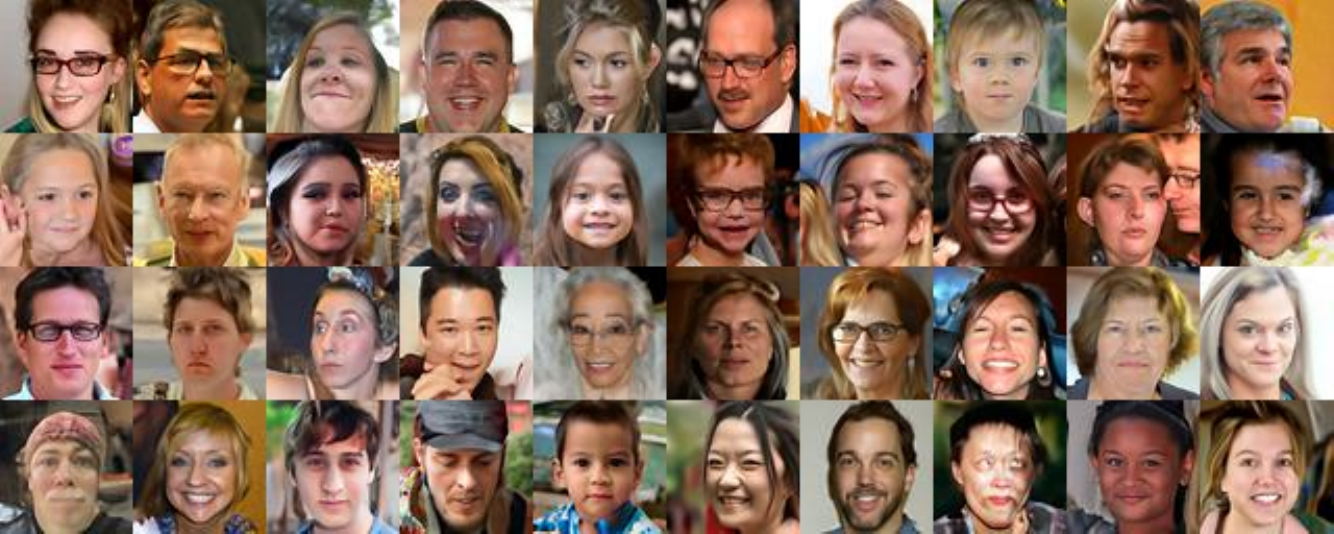}
      \caption{\(\ours\). NFE=\(\nfea\), FID = 8.05}
  \end{subfigure}
  \begin{subfigure}[b]{0.48\linewidth}
        \setlength{\abovecaptionskip}{0.03cm}
      \includegraphics[width=\linewidth]{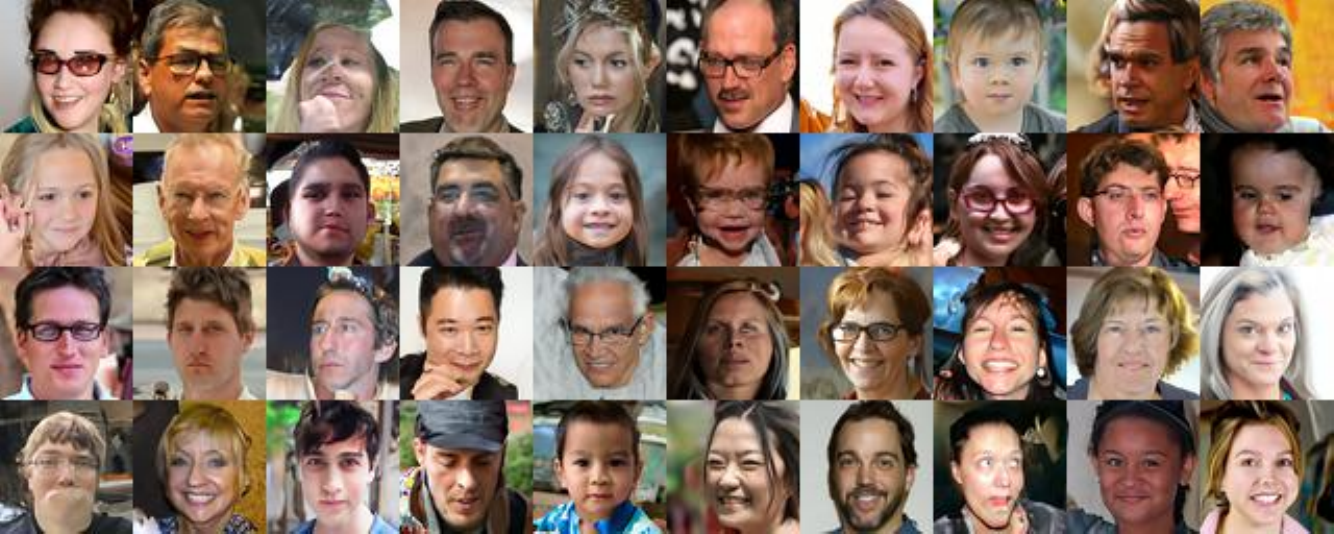}
      \caption{\(\ours\). NFE=\(\nfeb\), FID = 4.19}
  \end{subfigure}
  \setlength{\belowcaptionskip}{-0.35cm}
  \caption{Qualitative result on FFHQ 64$\times$64 (\(\nfea\) and \(\nfeb\) NFEs)}
  \label{fig:sup_grid_cifar10_3}
\end{figure*}

\begin{figure*}
  \centering
  \setlength{\abovecaptionskip}{-0.03cm}
  \begin{subfigure}[b]{0.48\linewidth}
  \setlength{\abovecaptionskip}{0.03cm}
      \includegraphics[width=\linewidth]{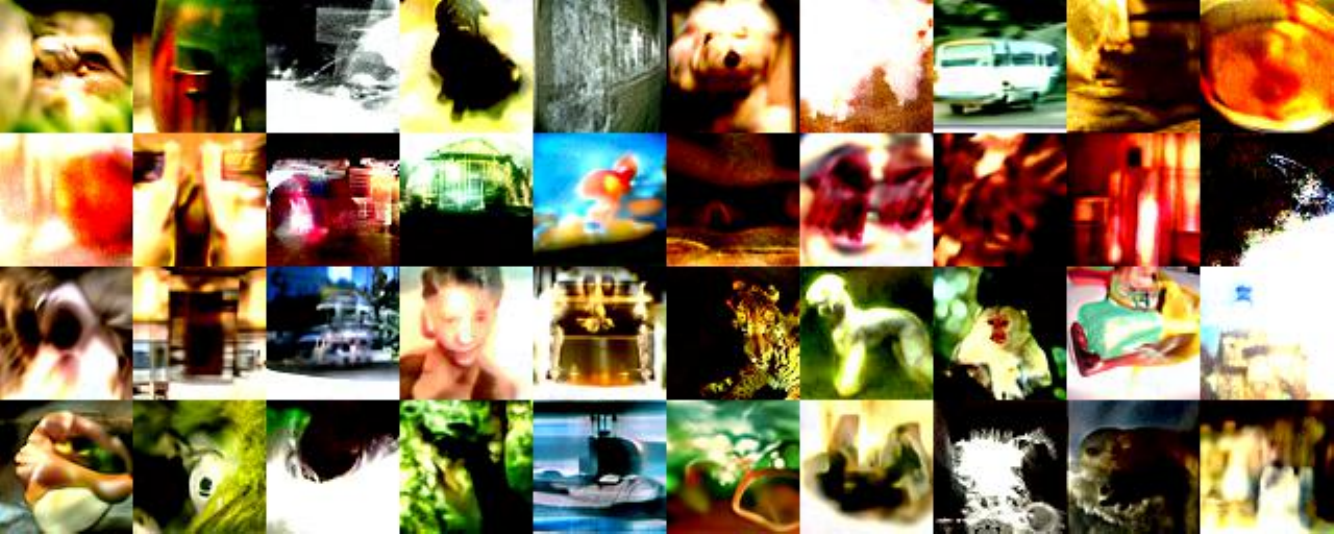}
      \caption{DPM-Solver-2. NFE=\(\nfea\), FID = 59.47}
  \end{subfigure}
  \begin{subfigure}[b]{0.48\linewidth}
  \setlength{\abovecaptionskip}{0.03cm}
      \includegraphics[width=\linewidth]{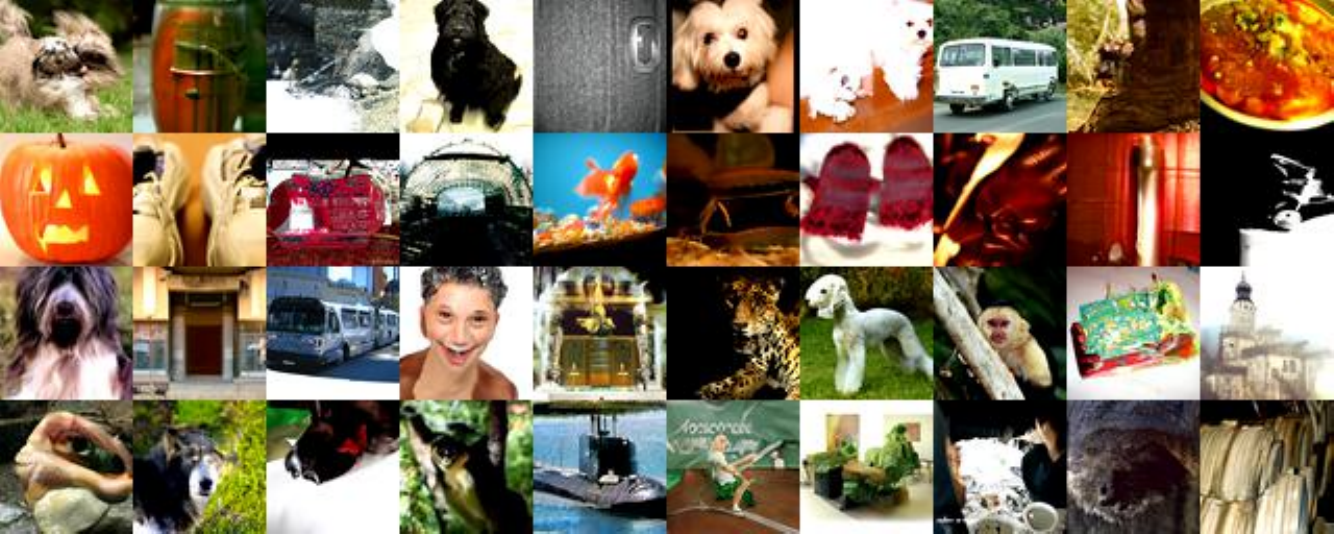}
      \caption{DPM-Solver-2. NFE=\(\nfeb\), FID = 11.31}
  \end{subfigure}
  \begin{subfigure}[b]{0.48\linewidth}
  \setlength{\abovecaptionskip}{0.03cm}
      \includegraphics[width=\linewidth]{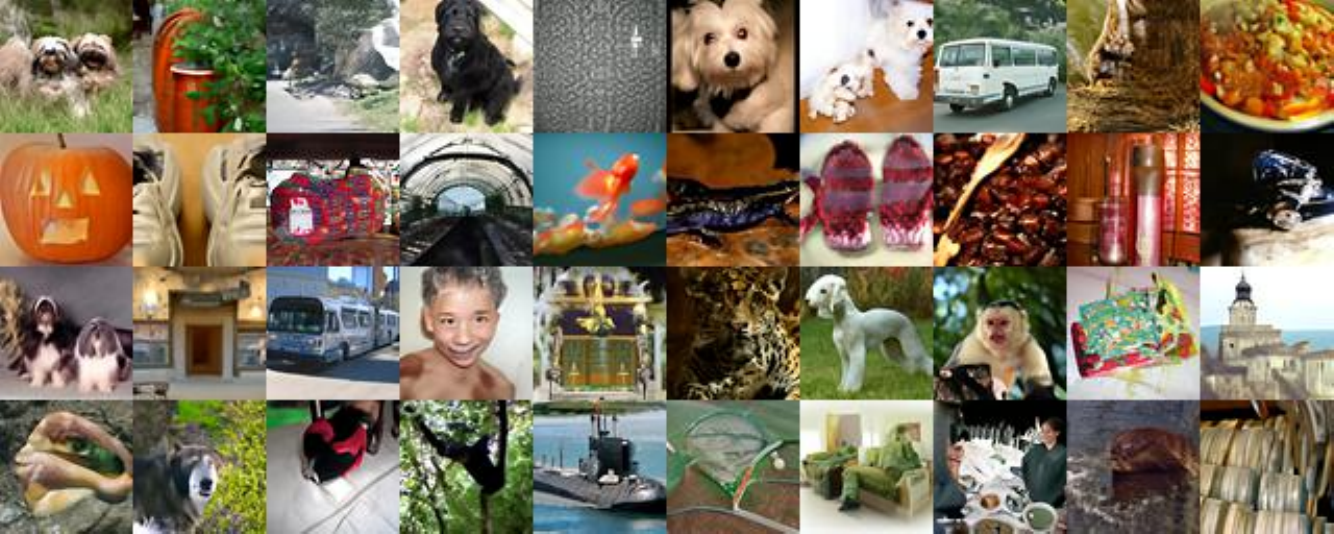}
      \caption{\(\ours\). NFE=\(\nfea\), FID = 6.90}
  \end{subfigure}
  \begin{subfigure}[b]{0.48\linewidth}
  \setlength{\abovecaptionskip}{0.03cm}
      \includegraphics[width=\linewidth]{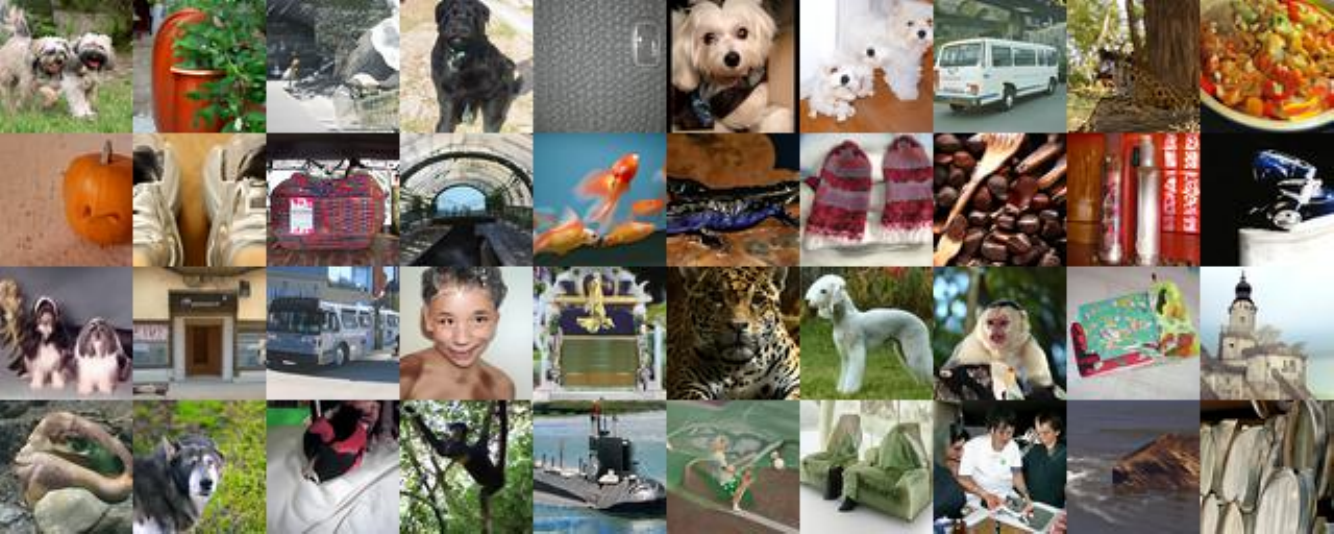}
      \caption{\(\ours\). NFE=\(\nfeb\), FID = 4.59}
  \end{subfigure}
  \caption{Qualitative result on ImageNet 64$\times$64 (\(\nfea\) and \(\nfeb\) NFEs)}
  \label{fig:sup_grid_cifar10_4}
\end{figure*}

\clearpage


\end{document}